\documentclass{article}

\usepackage{times}
\usepackage{graphicx} 

\usepackage{natbib}

\usepackage{algorithm}
\usepackage{algorithmic}

\usepackage{hyperref}

\usepackage{multirow}
\usepackage{subcaption}

\usepackage{amssymb}
\usepackage{amsmath}
\usepackage{amsthm}
\newtheorem{theorem}{Theorem}
\newtheorem{lemma}[theorem]{Lemma}
\newtheorem*{mydef}{Definition}

\newcommand{\uh}{\ensuremath{\hat{u}}}

\newcommand{\us}{\ensuremath{u^\ast}}
\newcommand{\vs}{\ensuremath{v^\ast}}

\newcommand{\Uc}{\ensuremath{\mathcal{U}}}
\newcommand{\Vc}{\ensuremath{\mathcal{V}}}

\usepackage[accepted]{icml2018}

\icmltitlerunning{$K$-Beam Minimax: Efficient Optimization for Deep Adversarial Learning}

\begin{document} 

\twocolumn[
\icmltitle{$K$-Beam Minimax: Efficient Optimization for Deep Adversarial Learning}



\icmlsetsymbol{equal}{*}

\begin{icmlauthorlist}
\icmlauthor{Jihun Hamm}{1}
\icmlauthor{Yung-Kyun Noh}{2}
\end{icmlauthorlist}

\icmlaffiliation{1}{The Ohio State University, Columbus, OH, USA.}
\icmlaffiliation{2}{Seoul National University, Seoul, Korea}
\icmlcorrespondingauthor{Jihun Hamm}{hammj@cse.ohio-state.edu}

\icmlkeywords{minimax optimization, saddle point, GAN, subgradient-descent}

\vskip 0.3in
]



\printAffiliationsAndNotice{}  

\begin{abstract}
Minimax optimization plays a key role in adversarial training of machine
learning algorithms, such as learning generative models,
domain adaptation, privacy preservation, and robust learning.
In this paper, we demonstrate the failure of alternating gradient descent
in minimax optimization problems due to the discontinuity of solutions of 
the inner maximization. 
To address this, we propose a new $\epsilon$-subgradient descent algorithm 
that addresses this problem by simultaneously tracking $K$ candidate solutions.
Practically, the algorithm can find solutions that previous saddle-point
algorithms cannot find, with only a sublinear increase of complexity in $K$. 
We analyze the conditions under which the algorithm converges to
the true solution in detail. 
A significant improvement in stability and convergence speed of the algorithm
is observed in simple representative problems, GAN training, and domain-adaptation problems.
\end{abstract}

\section{Introduction}\label{sec:intro}

There is a wide range of problems in machine learning which can be formulated
as continuous minimax optimization problems. 
Examples include generative adversarial nets (GANs)
\cite{goodfellow2014generative}, 
privacy preservation \cite{hamm2015preserving,edwards2015censoring},
domain adaption \cite{ganin2015unsupervised}, 
and robust learning \cite{globerson2006nightmare} to list a few.
More broadly, the problem of finding a worst-case solution or an equilibrium of
a leader-follower game \cite{bruckner2011stackelberg} can be formulated as 
a minimax problem. 
Furthermore, the KKT condition for a convex problem can be considered 
a minimax point of the Lagrangian \cite{arrow1958studies}. 
Efficient solvers for minimax problems can have positive impacts on all
these fronts.

To define the problem, consider a real-valued function $f(u,v)$ on 
a subset $\Uc \times \Vc \subseteq \mathbb{R}^d \times \mathbb{R}^D$. 
A (continuous) minimax optimization problem is 
$\min_{u\in\Uc} \max_{v \in \Vc} f(u,v)$\footnote{A more general problem is $\inf_{u\in\Uc} \sup_{v \in \Vc} f(u,v)$,
but we will assume that the min and the max exist and are achievable,
which are explained further in Sec.~\ref{sec:backgrounds}.}.
 It is called a discrete minimax problem if the maximization domain $\Vc$ is finite. 
A related notion is the (global) saddle point $(\us,\vs)$ which is a point that satisfies
\[
f(\us,v) \leq f(\us,\vs) \leq f(u,\vs),\;\;\forall (u,v) \in \Uc \times \Vc.
\]
When $f(u,v)$ is convex in $u$ and concave in $v$, saddle points coincide with
minimax points, due to the Von Neumann's theorem \cite{v1928theorie}:
$
\max_v \min_u f(u,v) = f(\us,\vs) = \min_u \max_v f(u,v).
$
The problem of finding saddle points has been studied intensively since the seminal
work of \citet{arrow1958studies}, and a gradient descent method was proposed by
\citet{uzawa1958iterative}. Much theoretical work has ensued,
in particular on the stability of saddle points and convergence
(see Sec.~\ref{sec:related work}).
However, the cost function $f$ in realistic machine learning applications is
seldom convex-concave
and may not have a global saddle point. 
Fig.~\ref{fig:simple} shows motivating examples of surfaces on 
$[-0.5,0.5]^2\subseteq \mathbb{R}^2$.
Examples (a),(b), and (c) are saddle point problems: all three have 
a critical point at the origin $(u,v)=(0,0)$, which is also a saddle point and a
minimax point.
However, examples (d),(e), and (f) do not have global saddle points:
example (d) has minimax points $(u,v)\in\{(\pm 0.25, -0.25),(\pm 0.25, 0.5)\}$ and 
examples (e) and (f) have minimax points $(u,v)=(0, \pm 0.5)$. 
These facts are not obvious until one analyzes each surface. 
(See Appendix for more information.)
Furthermore, the non-existence of saddle points also happens with unconstrained
problems: consider the function $f(u,v) = -0.5u^2 +2uv -v^2$ defined on $\mathbb{R}^2$. 
The inner maximum has the closed-form solution $\phi(u) = 0.5u^2$, and the outer
minimum $\min_u \phi(u)$ is $0$ at $u=0$. Therefore, $(u,v)=(0,0)$ is the global
minimax point (and also a critical point). 
However, $f$ cannot have a saddle point, local or global, since $f$ is strictly
concave in $u$ and $v$ respectively. 

\begin{figure*}[thb]
\centering
\begin{subfigure}{.3\linewidth}
	\centering
	\includegraphics[width=0.8\linewidth]{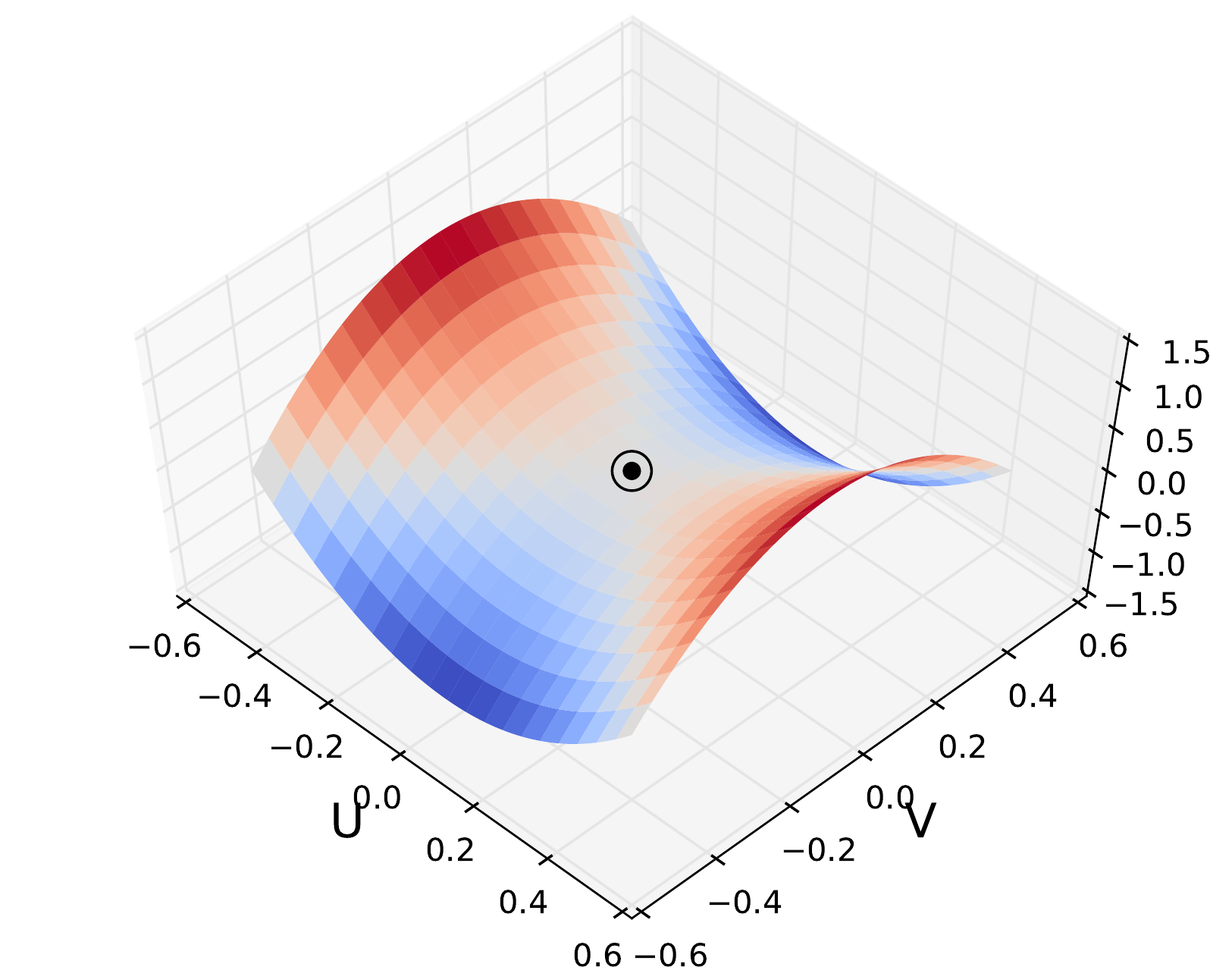}
	\caption{Saddle ($u^2-v^2$)}
\end{subfigure}
\begin{subfigure}{.3\linewidth}
	\centering
	\includegraphics[width=0.8\linewidth]{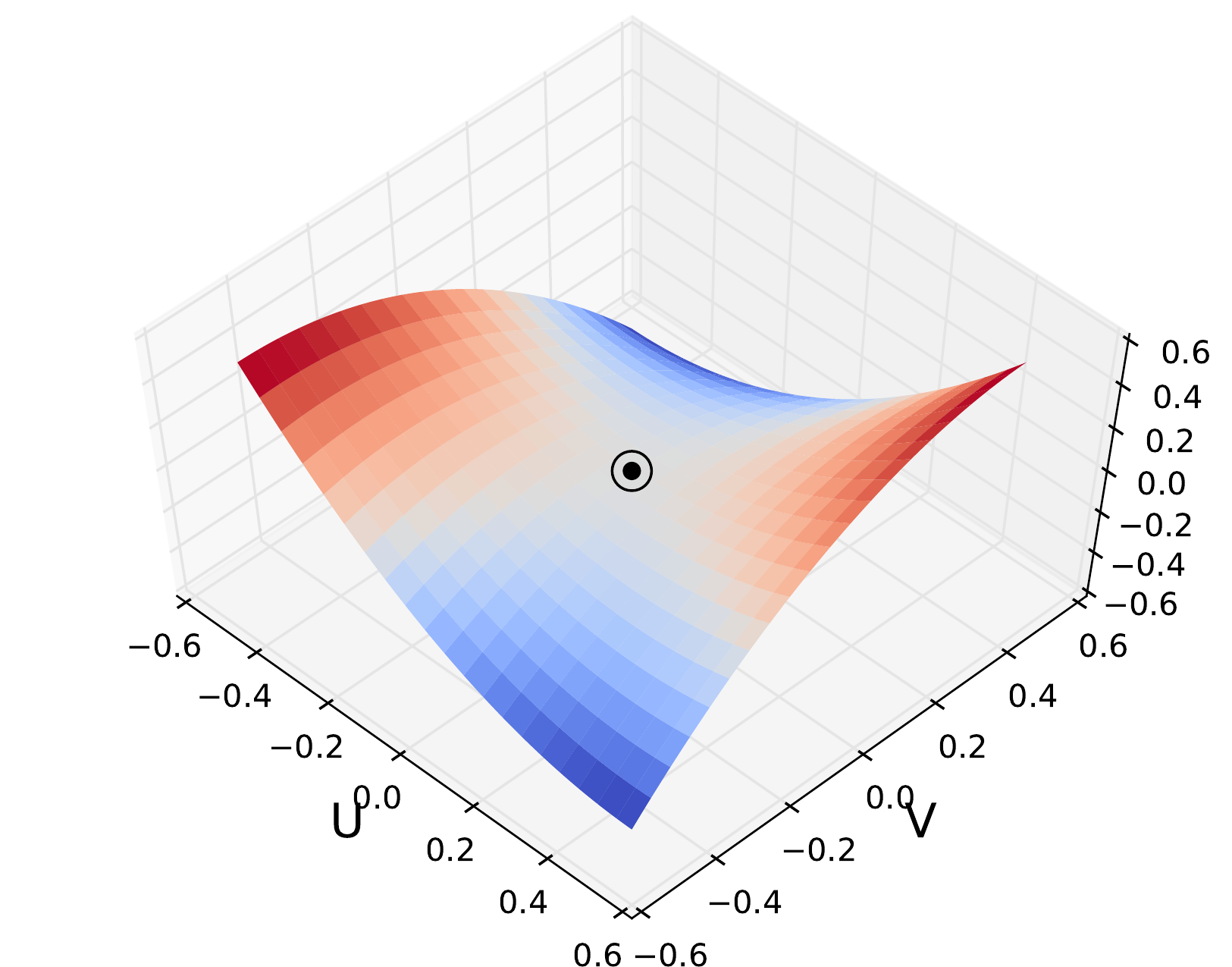}
	\caption{Rotated saddle ($u^2-v^2+2uv$)}
\end{subfigure}
\begin{subfigure}{.3\linewidth}
	\centering
	\includegraphics[width=0.8\linewidth]{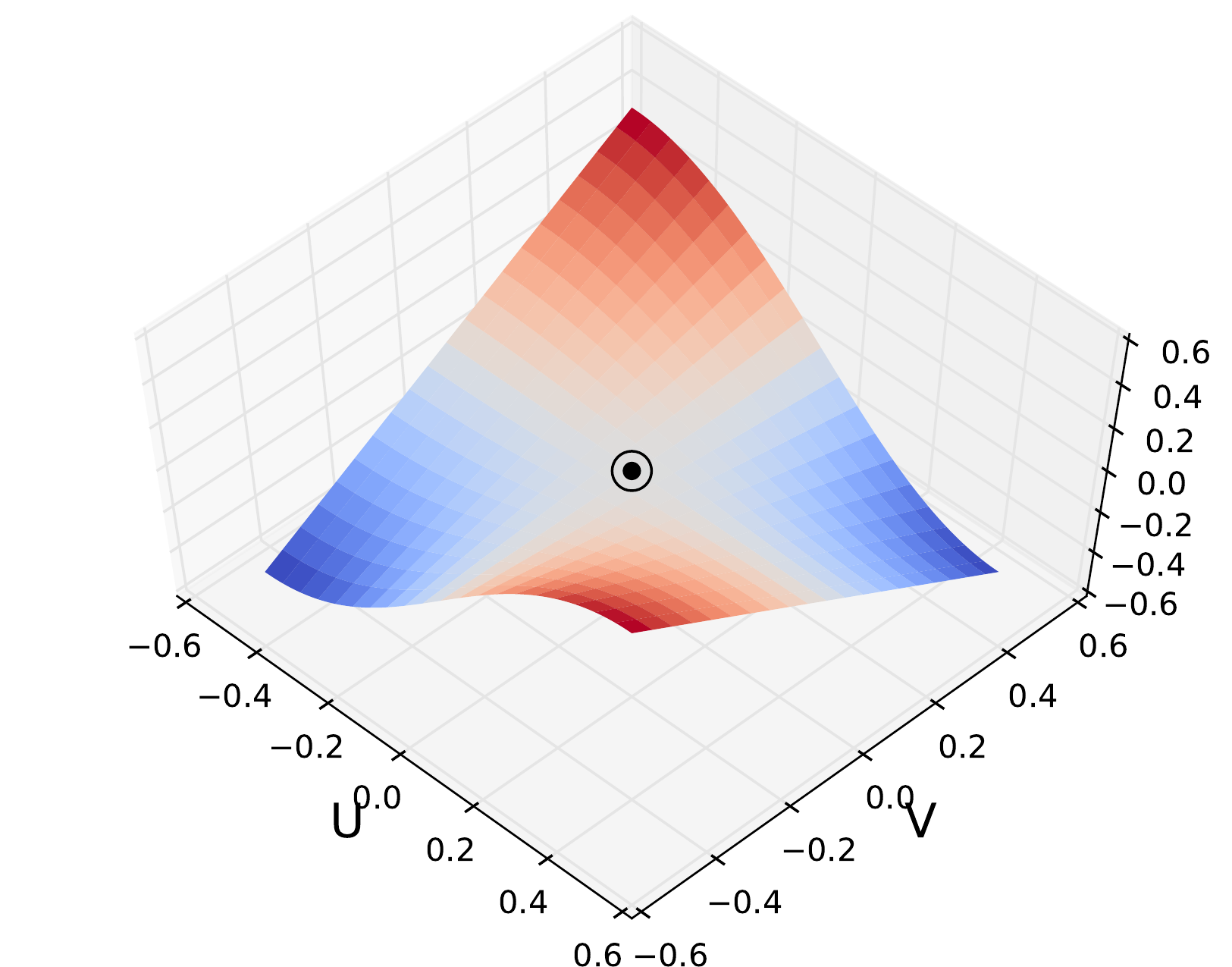}
	\caption{Seesaw ($-v\sin(\pi u)$)}
\end{subfigure}
\begin{subfigure}{.3\linewidth}
	\centering
	\includegraphics[width=0.8\linewidth]{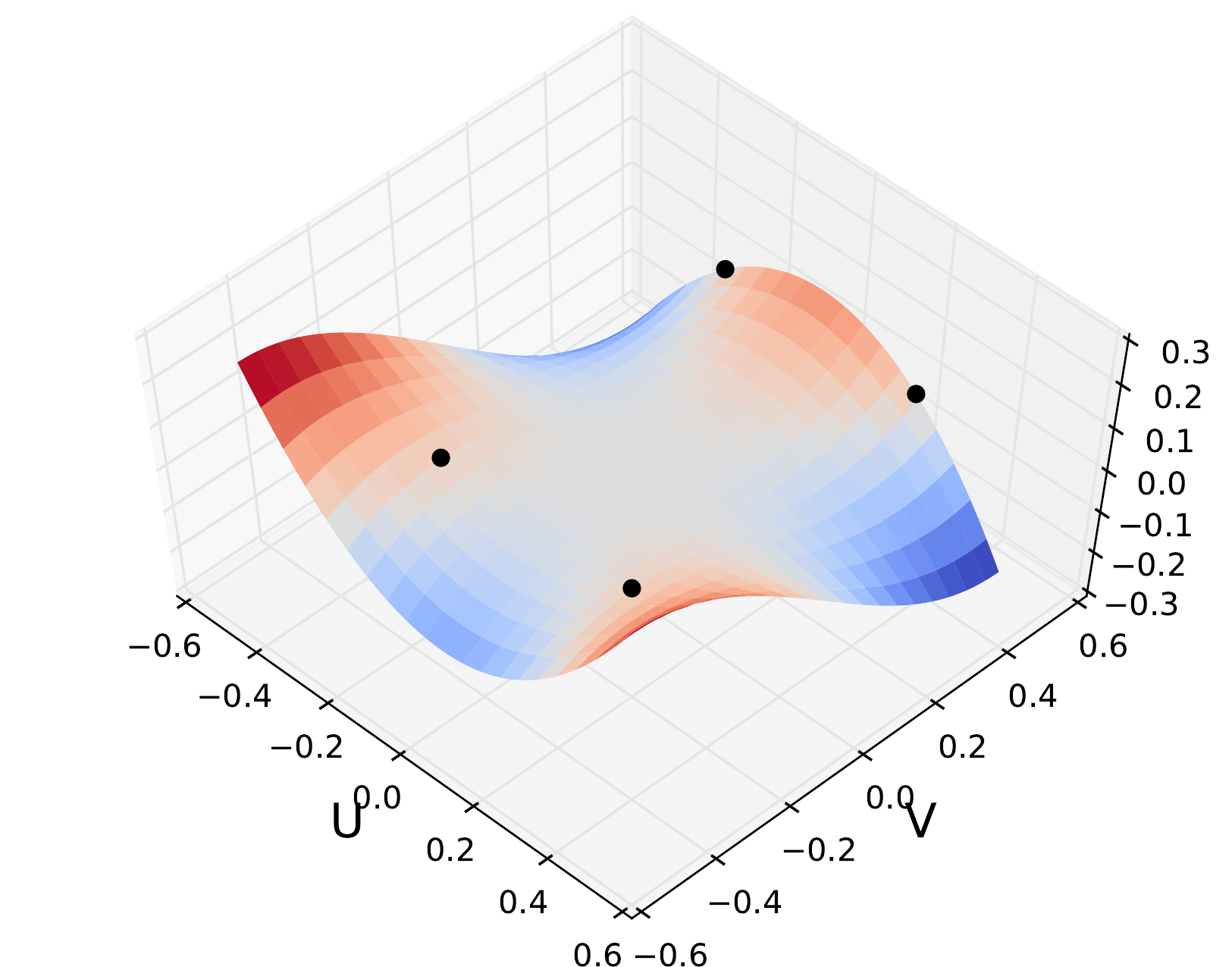}
	\caption{Monkey saddle ($v^3-3vu^2$)}
\end{subfigure}
\begin{subfigure}{.3\linewidth}
	\centering
	\includegraphics[width=0.8\linewidth]{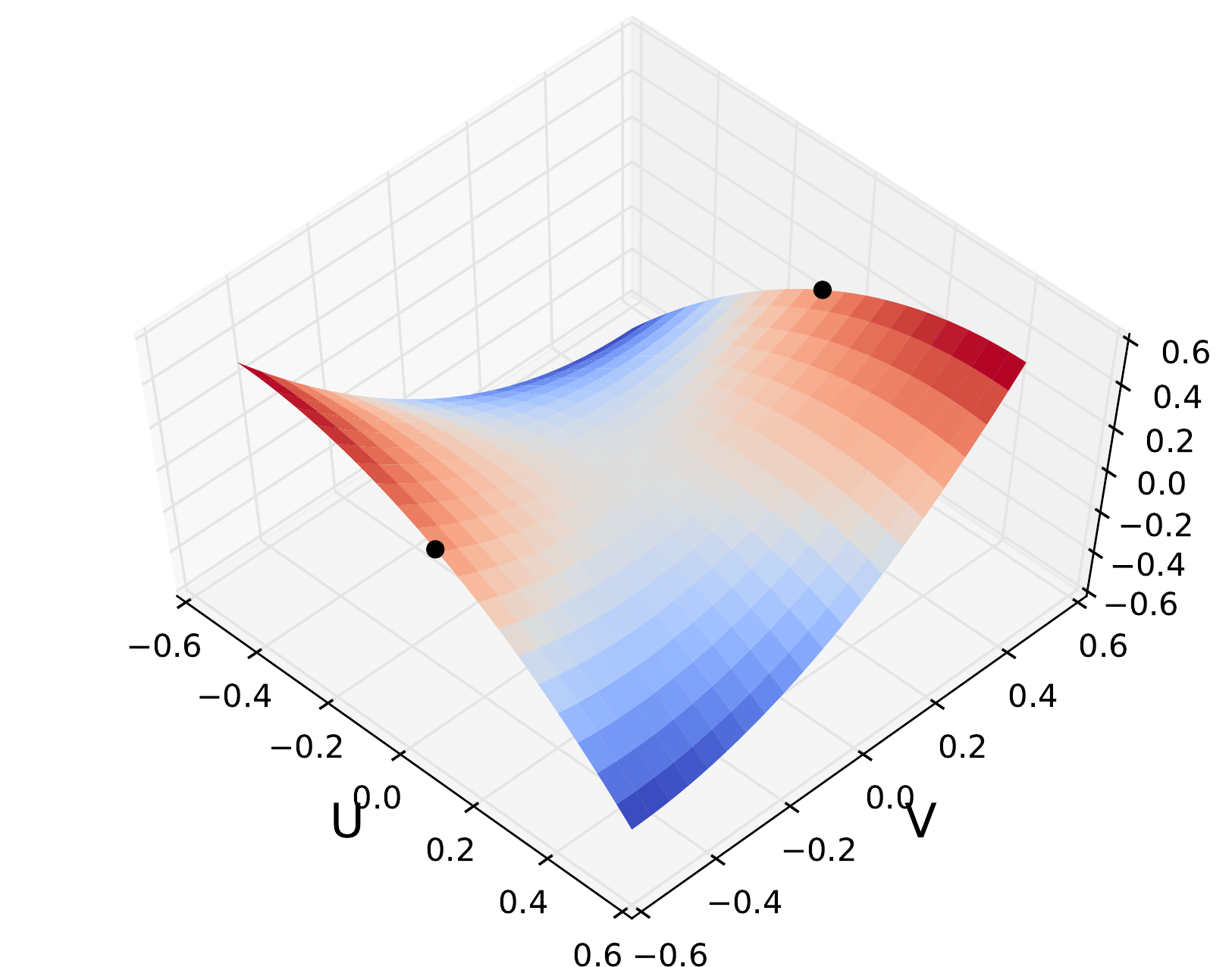}
	\caption{Anti-saddle ($-u^2+v^2+2uv$)}
\end{subfigure}
\begin{subfigure}{.3\linewidth}
	\centering
	\includegraphics[width=0.8\linewidth]{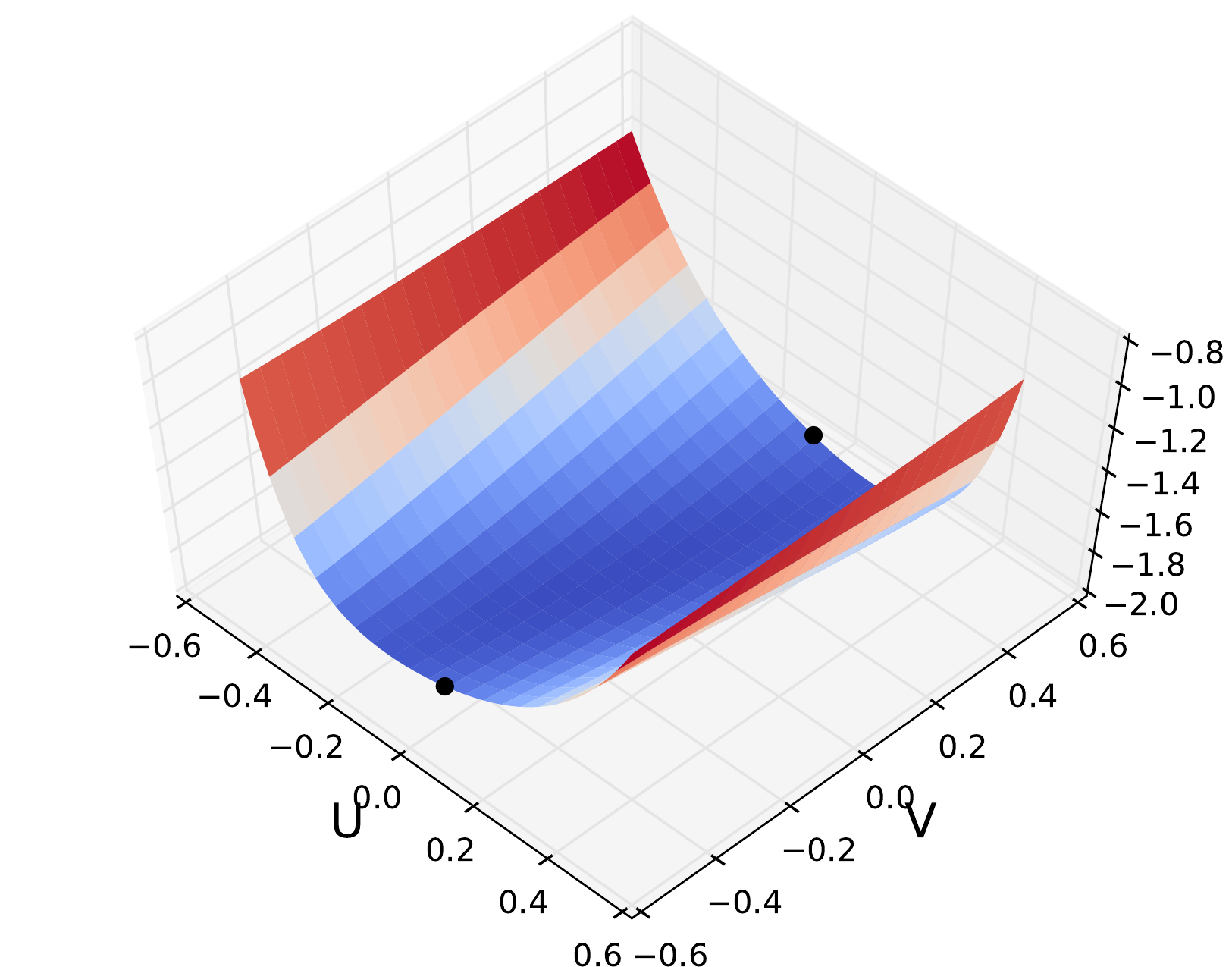}
	\caption{Weapons \citep{danskin1967theory} \\
	($ e^{-10(u+.5)e^{-(v+.5)}}+e^{-10(.5-u)e^{v-.5}}$)}
\end{subfigure}
\caption{Saddle points (empty black circles) and minimax points (filled black circles).
Surfaces (a),(b),(c) have both a saddle and minimax point at $(0,0)$,
whereas surfaces (d),(e),(f) do not have global saddle points but only minimax points.
All surfaces have a critical point at $(0,0)$.
}
\label{fig:simple}
\end{figure*}

Despite the fact that saddle points and minimax points are conceptually different, 
many machine learning applications in the literature do not distinguish the two. 
This is a mistake, because a local saddle point is only an equilibrium point and 
is not the robust or worst case solution that problem may be seeking. 
Furthermore, most papers have used the alternating gradient descent method
\begin{equation}\label{eq:alternating}
u \leftarrow u - \rho \nabla_u f(u,v),\;\;\mathrm{and}\;\; v \leftarrow v + \eta \nabla_v f(u,v).
\end{equation}
Alternating descent fails to find minimax points even for 2-dimensional examples (d)-(f)
in Fig.~\ref{fig:simple} as we show empirically in the Sec.~\ref{sec:exp simple}.
To explain the reason for failure, let's define the inner {maximum value}  
$\phi(u):=\max_{v\in \Vc} f(u,v)$  
and the corresponding maximum points $R(u):=\arg\max_{v\in \Vc} f(u,v)$. 
The main reason for failure is that the solution $R(u)$ may not be unique and can
be discontinuous w.r.t. $u$. For example, in Fig.~\ref{fig:simple} (e), we have
$R(u)=\{-0.5\}$ for $u<0$ and $R(u)=\{+0.5\}$ for $u>0$. 
This discontinuity at $u=0$ makes it impossible for a gradient descent-type method
to keep track of the true inner maximization solution as $v$ has to jump between 
$\pm 0.5$.\footnote{Also note that a 
gradient descent-type algorithms will diverge away from $(0,0)$ which is 
an anti-saddle, i.e., $f$ is concave-convex at $(0,0)$ instead of convex-concave.}

In this paper, we propose a $K$-beam approach that tracks $K$ candidate solutions 
(or ``beams'') of the inner maximization problem to handle the discontinuity. 
The proposed $\epsilon$-subgradient algorithm (Algs.~\ref{alg:proposed} and ~\ref{alg:descent direction})
generalizes the alternating gradient-descent method ($K$=1) and 
also exact subgradient methods. 
In the analysis, we prove that it can find minimax points if the inner problem
$\max_{v\in\Vc} f(u,v)$ can be approximated well by $\max_{v \in A} f(u,v)$
over a finite set $A$ at each $u$, summarized by Theorem~\ref{thm:main}
which is the main result of analysis. 
For the purpose of analysis we assume that $f$ is convex in $u$ similar to 
the majority of the analyses on gradient-type algorithms. 
However, we allow $f$ to be non-concave in $v$ and have multiple local maxima,
which makes our setting much more general than that of classic saddle point problems
with convex-concave $f$ or previous work which assumed only bilinear couplings 
between $u$ and $v$ \cite{chambolle2011first,he2012convergence}. 

Practically, the algorithm can find solutions that gradient descent
cannot find with only a sublinear increase of time complexity in $K$. 
To demonstrate the advantages of the algorithm, we test the algorithm on 
the toy surfaces (Fig.~\ref{fig:simple}) for which we know the true minimax solutions.
For real-world demonstrations, we also test the algorithm on GAN problems \cite{goodfellow2014generative},
and unsupervised domain-adaptation problems \cite{ganin2015unsupervised}. %
Examples were chosen so that the performance can be measured objectively -- by 
the Jensen-Shannon divergence for GAN and by cross-domain classification error 
for domain adaptation. 
Evaluations show that the proposed $K$-beam subgradient-descent approach
can significantly improve stability and convergence speed of minimax optimization. 

The remainder of the paper is organized as follows.
We discuss related work in Sec.~\ref{sec:related work} and backgrounds in Sec.~\ref{sec:backgrounds}.
We propose the main algorithm in Sec.~\ref{sec:minimax}, and present the analysis
in Sec.~\ref{sec:analysis}. 
The results of experiments are summarized in Sec.~\ref{sec:experiments}, and
the paper is concluded in Sec.~\ref{sec:conclusions}.
Due to space limits, all proofs in Sec.~\ref{sec:analysis} and additional figures
are reported in Appendix. 
The codes for the project can be found at
{\url{https://github.com/jihunhamm/k-beam-minimax}.}

\section{Related work}\label{sec:related work}


Following the seminal work of \citet{arrow1958studies}
(Chap.~10 of \citet{uzawa1958iterative} in particular), many researchers
have studied the questions of the convergence of (sub)gradient descent 
for saddle point problems under different stability conditions 
\cite{dem1972numerical,golshtein1972generalized,maistroskii1977gradient,
zabotin1988subgradient,nedic2009subgradient}.  
Optimization methods for minimax problems have also been studied somewhat independently. 
The algorithm proposed by \citet{salmon1968minimax}, referred to as the Salmon-Daraban
method by \citet{dem1972numerical}, 
finds continuous minimax points by solving successively larger discrete minimax
problems. The algorithm can find minimax points for a differentiable $f$ on
compact $\Uc$ and $\Vc$. 
However, the Salmon-Daraban method is impractical, as its requires 
exact minimization and maximization steps at each iteration, 
and also because the memory footprint increases linearly with iteration.
Another method of continuous minimax optimization was proposed by 
 \citet{dem1971theory,dem1974introduction}.  
The grid method, similar to the Salmon-Daraban method, iteratively solves a
discrete minimax problem to a finite precision using the $\epsilon$-steepest
descent method. 

Recently, a large number of papers tried to improve GAN models in particular 
by modifying the objective (e.g.,\citet{uehara2016generative,nowozin2016f,arjovsky2017wasserstein}),
but relatively little attention was paid to the improvement of the 
optimization itself. 
Exceptions are the multiadversarial GAN \cite{durugkar2016generative}, 
and the Bayesian GAN \cite{saatci2017bayesian}, both of which used multiple 
discriminators and have shown improved performance, although no analysis was provided. 
Also, gradient-norm regularization has been studied recently to stabilize
gradient descent \cite{mescheder2017numerics,nagarajan2017gradient,roth2017stabilizing},
which is orthogonal to and can be used simultaneously with the proposed method.
Note that there can be multiple causes of instability in minimax optimization,
and what we address here is more general and not GAN-specific.

\section{Backgrounds}\label{sec:backgrounds}

Throughout the paper, we assume that
$f(u,v): \Uc \times \Vc \to \mathbb{R}$ is a continuously differentiable function
in $u$ and $v$ separately. 
A general form of the minimax problem is $\inf_{u\in\Uc} \sup_{v \in \Vc} f(u,v)$.
We assume that $\Uc$ and $\Vc$ are compact and convex subsets of Euclidean spaces
such as a ball with a large but finite radius.
Since $f$ is continuous, min and max values are bounded and attainable. 
In addition, the solutions to min or max problems are assumed to be 
in the interior of $\Uc$ ad $\Vc$, enforced by adding appropriate
regularization (e.g, $\|u\|^2$ and $-\|v\|^2$) to the optimization problems if necessary. 

As already introduced in Sec.~\ref{sec:intro}, the inner maximum value and points 
are the key objects in the analysis of minimax problems.
\begin{mydef}
The {maximum value} $\phi(u)$ is $\max_{v\in \Vc} f(u,v)$.
\end{mydef}
\begin{mydef}
The corresponding {maximum points} $R(u)$ is 
$\arg\max_{v\in \Vc} f(u,v)$, i.e., $R(u)=\{v \in \Vc\;|\; f(u,v)=\max_{v\in \Vc} f(u,v)\}$.
\end{mydef}
Note that $\phi(u)$ and $R(u)$ are functions of $u$. With abuse of notation,
the $R(\Uc)$ is the union of maximum points for all $u \in \Uc$, i.e., 
$R(\Uc):= \bigcup_{u \in \Uc} R(u)$

As a generalization, the $\epsilon$-maximum points $R^\epsilon(u)$ are the points 
whose values are $\epsilon$-close to the maximum: \\
$R^\epsilon(u):=\{v\in \Vc\;|\; \max_{v\in \Vc} f(u,v)- f(u,v)\leq \epsilon\}$.

\begin{mydef}
$S(u)$ is the set of local maximum points 
\begin{eqnarray*}
S(u)&:=&\{v_0\in \Vc\;|\;\exists r>0\;\; s.t.\;\; \forall v\in\Vc,\\
&&\;\;\;\;\;\;\|v_0-v\|\leq r \Rightarrow f(u,v_0)\geq f(u,v)\}.
\end{eqnarray*}
\end{mydef}
Note that $\nabla_v f(u,v)=0$ for $v \in S(u)$ due to differentiability assumption,
and that $R(u) \subseteq S(u)$. 

\begin{mydef}
$\min_{u\in\Uc} \max_{v\in A} f(u,v)$  is a discrete minimax problem if 
$A$ is a finite set $A:=\{v^1,...,v^K\} \subseteq \Vc$.
\end{mydef}
We accordingly define $\phi_A(u)$, $R_A(u)$ and $R^\epsilon_A(u)$ by
$\phi_A(u):=\max_{v \in A} f(u,v)$, and 
$R^\epsilon_A(u):=\{v\in A\;|\; \max_{v\in A} f(u,v)- f(u,v)\leq \epsilon\}$.

We also summarize a few results we will use, which can be found in convex analysis
textbooks such as \citet{hiriart2001fundamentals}.
\begin{mydef}
An $\epsilon$-subgradient of a convex function $\phi(u)$ at $u_0$ is $g\in \mathbb{R}^d$
that satisfies for all $u$
\[
\phi(u)-\phi(u_0) \geq \langle g,\; u - u_0\rangle - \epsilon. 
\]
The $\epsilon$-subdifferential $\partial_\epsilon \phi(u_0)$ is the set of all 
$\epsilon$-subgradients at $u_0$.
\end{mydef}

Consider the convex hull $\mathrm{co}\{\cdot\}$ of the set of gradients.
\begin{lemma}[Corollary 4.3.2, Theorem 4.4.2, \citet{hiriart2001fundamentals}]
Suppose $f(u,v)$ is convex in $u$ for each $v \in A$. 
Then $\partial \phi_A(u)=\mathrm{co}\{\cup_{v \in A} \nabla_u f(u,v)\}$.
Similarly, suppose $f(u,v)$ is convex in $u$ for each $v \in \Vc$. Then $\partial \phi(u)=\mathrm{co}\{\cup_{v \in \Vc} \nabla_u f(u,v)\}$.
\end{lemma}

\begin{mydef}
A point $u$ is called an $\epsilon$-stationary point of $\phi_{A}(u)$ if 
$\max_{v \in R^\epsilon_{A}} \langle \nabla_u f(u,v),\;g\rangle \geq 0$
for all $g \in \mathbb{R}^d$.
\end{mydef}


\begin{lemma}[Chap 3.6, \citet{dem1974introduction}]\label{lem:stationary}
A point $u$ is an $\epsilon$-stationary point of $\phi(u)$
if and only if $0 \in \mathrm{co}\{\cup_{v \in R^\epsilon(u)} \nabla_u f(u,v)\}$.
\end{lemma}

\section{Algorithm}\label{sec:minimax}

The alternating gradient descent method predominantly used in the literature
fails when the inner maximization $\max_{v\in\Vc} f(u,v)$ has more than
one solution, i.e., $R(u)$ is not a singleton.
To address the problem, we propose the $K$-beam method to simultaneously track
the maximum points $R(u)$ by keeping the candidate set $A=(v^1,...,v^K)$ for some large $K$. (The choice for $K$ will be discussed in Analysis and Experiments.) 
This approach can be exact, if the maximum points over the whole domain $R(\Uc)$ is finite,
as in examples (a),(e) and (f) of Fig.~\ref{fig:simple} (see Appendix.)
In other words, the problem becomes a discrete minimax problem. 
More realistically, the maximum points $R(\Uc)$ is infinite but $R(u)$ can 
still be finite for each $u$, as in all the examples of Fig.~\ref{fig:simple} except (c). 
At $i$-th iteration, the $K$-beam method updates the current candidates
$A_i=(v^1_i,...,v^K_i)$ such that the discrete maximum $\phi_{A_i}(u)$ 
is a good approximation to the true $\phi(u)$. 
In addition, we present an $\epsilon$-subgradient algorithm that generalizes
exact subgradient algorithms. 

\subsection{Details of the algorithm}

\begin{algorithm}[htb] \caption{$K$-beam $\epsilon$-subgradient descent} \label{alg:proposed}
{Input}: $f, K, N, (\rho_i), (\eta_{i}), (\epsilon_i)$\\
{Output}: $u_N, A_N$ \\
Initialize $u_0, A_0=(v^1_0,...,v^K_0)$\\
{Begin}
\begin{algorithmic}
\FOR{$i=1,\;...\;,N$}
	\STATE\hspace{-0.075in}{{\it Min step:}}
	\STATE{Update $u_i = u_{i-1} + \rho_i\; g(u_i,A_i,\epsilon_i)$ where
	$g$ is a descent direction from Alg.~\ref{alg:descent direction}.}	
	\STATE\hspace{-0.075in}{{\it Max step:}}
	\FOR{$k=1,\;...\;,K$ in parallel}
		\STATE{Update $v_i^k \leftarrow  v_{i-1}^{k} + \eta_{i}\; \nabla_v f(u_i,v_{i-1}^k)$.}		
	\ENDFOR
	\STATE{Set $A_i = (v^1_i,\;...\;,v^K_i)$.}
\ENDFOR
\end{algorithmic}
\end{algorithm}

Alg.~\ref{alg:proposed} is the main algorithm for solving minimax problems. 
At each iteration, the algorithm alternates between the min step and the max step.
In the min step, it approximately minimizes $\phi(u)$ by following a
subgradient direction $z \in \partial_{\epsilon} \phi_{A_i}(u)$.
In the max step, it updates $A_i=\{v^1_i,...,v^K_i\}$ to track the local
maximum points of $f(u,\cdot)$ so that the approximate subdifferential
$\partial_{\epsilon} \phi_{A_i}(u)$ remains
close to the true subdifferential $\partial \phi(u)$.

The hyperparameters of the algorithm are the beam size ($K=|A|$), 
the total number of iterations ($N$),
and the step size schedules for min step $(\rho_i)$ and for max step $(\eta_{i})$
and the approximation schedule $(\epsilon_i)$.

\begin{algorithm}[htb] \caption{Descent direction} \label{alg:descent direction}
{Input}: $f, u, A=(v^1,...,v^K), \epsilon$\\
{Output}: $g$ \\
{Begin}
\begin{algorithmic}
\STATE{Find $k_{\max} = \arg\max_{1 \leq k \leq K} f(u,v^k)$.}
\STATE{Find $\{v^{k_1},...,v^{k_n}\}=R^{\epsilon}_A(u)=\{v\in A\;|\; f(u,v^{k_{\max}})-f(u,v^k)\leq \epsilon \}$.} 
\STATE{Compute $z_j = \nabla_u f(u,v^{k_j})$ for $j=1,..,n$.}
\STATE\hspace{-0.075in}{\it Optional stopping criterion:}
\IF{$0 \in \mathrm{co}\{ z_1 \cup ... \cup z_n\}$}
\STATE{Found a stationary point. Quit optimization.}
\ENDIF
\STATE\hspace{-0.075in}{\it Decent direction:}
\IF{$n=1$}
\STATE{Return $g=-z_1$.}
\ELSE
\STATE{Randomly choose $z \in \mathrm{co}\{ z_1 \cup ... \cup z_n\}$ and return 
$g= -z$.}
\ENDIF
\end{algorithmic}
\end{algorithm}

Alg.~\ref{alg:descent direction} is the subroutine for finding a descent direction. 
If $\epsilon$=0, this subroutine identifies the best candidate $v^{k_\mathrm{max}}$
among the current set $A_i$ and returns its gradient $\nabla_u f(u,v^{k_\mathrm{max}})$.
If $\epsilon>0$, it finds $\epsilon$-approximate candidates and returns any direction in
the convex hull of their gradients.
We make a few remarks below. 
\begin{itemize}
\itemsep0em 
\item Alternating gradient descent (\ref{eq:alternating}) is a special case of
the $K$-beam algorithm for $K=1$ and $\epsilon_1=\epsilon_2=...=0$.
\item As will be shown in the experiments, the algorithm usually performs better 
with increasing $K$. However, increase in computation can be made negligible, 
since the $K$ updates in the max step can be performed in parallel.
\item One can use different schemes for the step sizes $(\rho_i), (\eta_{i})$ and
$(\epsilon_i)$. For the purpose of analysis, we use non-summable but 
square-summable step size, e.g., $1/i$. Any decreasing sequence $(\epsilon_i)\to 0$
can be used. 
\item The algorithm uses subgradients since the maximum value $\phi(u)$ is 
non-differentiable even if $f$ is, when there are more than one maximum point
$(|R(u)|>1)$ \cite{danskin1967theory}.
In practice, when $\epsilon$ is close to 0, the approximate
maximum set $R^\epsilon_A(u)$ in Alg.~\ref{alg:descent direction} is often a singleton
in which case the descent direction from Alg.~\ref{alg:descent direction}
is simply the gradient $-\nabla_u f(u,v)$. 
\item The convergence of the algorithm (Sec.~\ref{sec:analysis}) is not affected by 
the random choice $z\in \mathrm{co}\{ z_1 \cup ... \cup z_n\}$ in 
Alg.~\ref{alg:descent direction}. 
In practice, the random choice can help to avoid local minima if $f$ is not convex. 

\item Checking the stopping criterion $0 \in \mathrm{co}\{\cup_j z_j\}$ can be 
non-trivial (see Sec.~\ref{sec:stopping criteria}), and may be
skipped in practice. 

\end{itemize}

\section{Analysis}\label{sec:analysis}

We analyze the conditions under which Alg.~\ref{alg:proposed} and 
Alg.~\ref{alg:descent direction} find a minimax point. 
We want the finite set $A_i$ at $i$-th iteration 
to approximate the true maximum points $R(u_i)$ well,
which we measure by the following two distances.
Firstly, we want the following one-sided Hausdorff distance
\begin{equation}\label{eq:dist1}
d_H(R(u_i),A_i):=\max_{v \in R(u_i)} \min_{v' \in A_i} \|v-v'\|
\end{equation}
to be small, i.e., each global maximum $v \in R(u_i)$ is close to at least one
candidate in $A_i$.
Secondly, we also want the following one-sided Hausdorff distance
\begin{equation}\label{eq:dist2}
d_H(A_i,S(u_i)):=\max_{v' \in A_i} \min_{v \in S(u_i)} \|v-v'\|
\end{equation}
to be small, where $S(u_i)$ is the local maxima, 
i.e., each candidate is close to at least one local maximum $v \in S(u_i)$.
This requires that $K$ is at least as large as $\max_u |S(u)|$. 

We discuss the consequences of these requirements more precisely in the rest
of the section. 
For the purpose of analysis, we will make the following additional assumptions.\\
{\bf Assumptions.}\emph{
$\phi(u)$ is convex and achieves the minimum $\phi^\ast=\phi(\us)$.
Also, $f(u,v)$ is $l$-Lipschitz in $v$ for all $u$, and 
$\nabla_u f(u,v)$ is $r$-Lipschitz in $v$ for all $u$.}

{\it Remark on the assumption.} 
Note that we only assume the convexity of $f$ over $u$ and not the concavity
over $v$, which makes this setting more general than that of classic analyses which
assume the concavity over $v$, or that of restricted models with a bilinear coupling
$f(u,v) = f_{\mathrm{convex}}(u) + g_{\mathrm{concave}}(v) + u^T A v$. 
While we allow $f$ to be non-concave in $v$ and have multiple local maxima, 
we also require $f$ and $\nabla_u f$ to be Lipschitz in $v$ for the purpose of analysis.  

\subsection{Finite $R(u)$, exact max step}
If $R(u)$ is finite for each $u$, and if the maximization in the max step can be done
exactly as assumed in the Salmon-Daraban method \cite{salmon1968minimax}, 
then the problem is no more difficult than a discrete minimax problem.
\begin{lemma}\label{lem:zero dist}
Suppose $R(u)$ is finite at $u$.
If $d_H(R(u),A)=0$, then $R(u) = R_{A}(u)$ and therefore 
$\partial \phi(u) = \partial \phi_{A}(u)$.
\end{lemma}
Since the subdifferential is exact, Alg.~\ref{alg:proposed}
finds a minimax solution as does the subgradient-descent method 
with the true $\phi(u)$. We omit the proof and present a more general theorem shortly. 

\begin{figure*}[thb]
\centering
\includegraphics[width=0.99\linewidth]{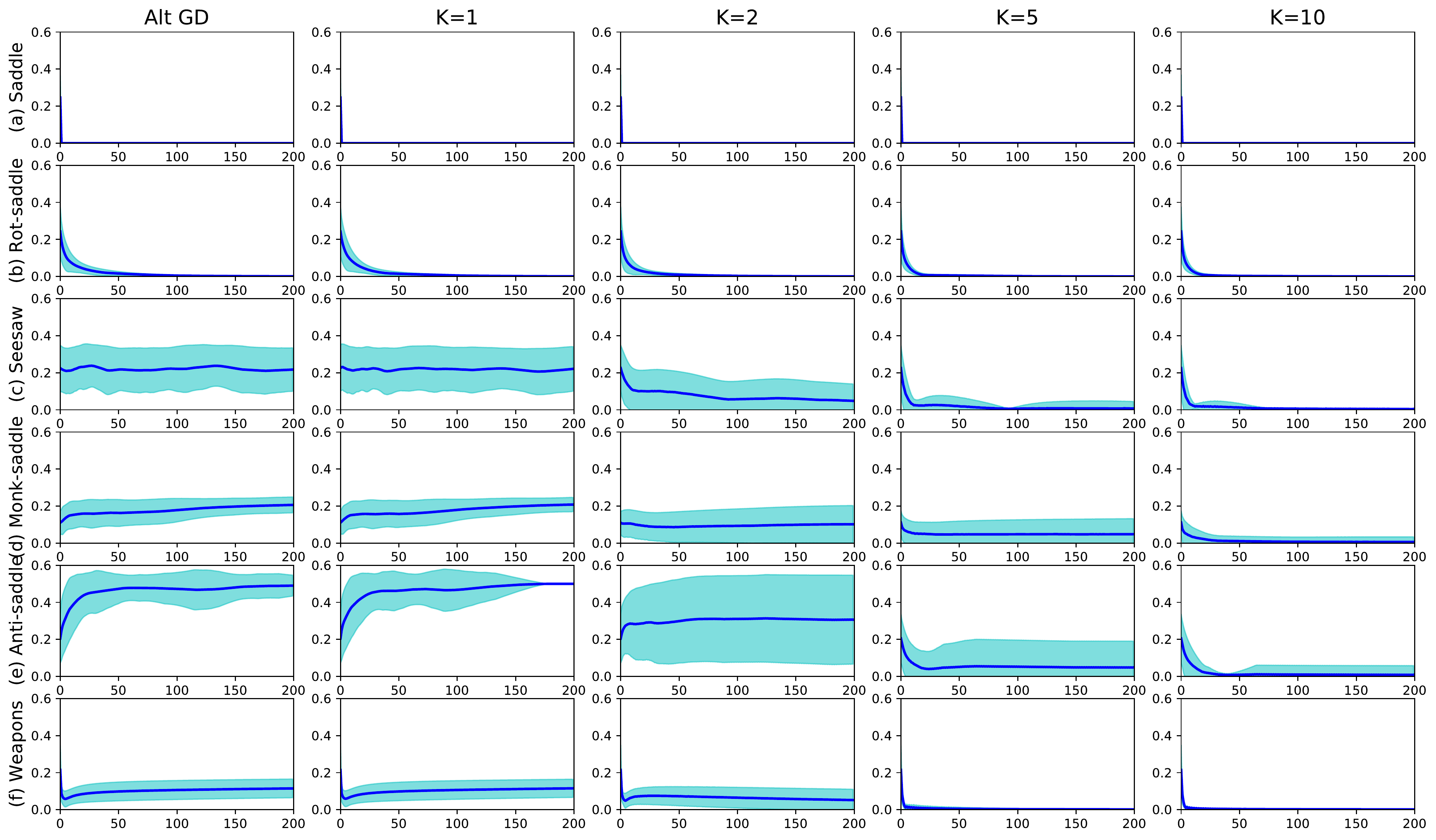}
\caption{
Convergence of Alt-GD and $K$-beam ($K=1,2,5,10$) for the six surfaces (Fig.~\ref{fig:simple})
after 200 iterations, measured by the distance to the closest optimal point.
The dark blue line is the average error and the light blue area is the avg$\pm$std. 
For easy surfaces (a) and (b), all methods converge quickly. 
For surfaces (c)-(f), Alt-GD fails to converge to the solution whereas $K$-beam does 
with $K>1$. 
}
\label{fig:simple convergence}
\end{figure*}

\subsection{Finite $R(u)$, inexact max step}

Exact maximization in each max step is unrealistic, unless $\max_v f(u,v)$ 
can be solved in closed form. Therefore we consider what happens to the convergence 
of the algorithm with an approximate max step. 
If $d_H(R(u),A)\leq \delta$ and $d_H(A,S(u))\leq \delta$ for some $\delta\geq 0$,
how close are $\phi(\cdot)$ and $\phi_A(\cdot)$ in the vicinity of $u$?
The following lemmas answer this question. 
(See Appendix for a visual aid.)
From the smoothness assumptions on $f$, we have
\begin{lemma}\label{lem:lipschitz}
If $d_H(R(u),A)\leq \delta$, then for each $v \in R(u)$ there is one or more
$v' \in A$ such that $\phi(u) - f(u,v') \leq l \delta$ and 
$\|\nabla_u f(u,v) - \nabla_u f(u,v')\| \leq r \delta$. 
\end{lemma}

The following lemma shows that if $A$ approximates $R(u)$ well, then
$v'$ chosen by Alg.~\ref{alg:descent direction} is not far from a true maximum 
$v\in R(u)$.
\begin{lemma}\label{lem:separation}
Assume $R(u)$ and $S(u)$ are both finite at $u$.
Let $\zeta = \phi(u) - \max_{v \in S(u)\setminus R(u)} f(u,v)$ be the
smallest gap between the global and the non-global maximum values at $u$. 
If all local maxima are global maxima, then set $\zeta=\infty$.
If $d_H(R(u),A)\leq \delta$ and $d_H(A,S(u))\leq \delta$ where
$\delta < 0.5(\zeta-\epsilon)/l$, 
then for each $v' \in R^\epsilon_A(u)$, there is
$v \in R(u)$ such that $\|v - v'\|\leq \delta$.
\end{lemma}
Furthermore, the subgradients at the approximate maximum points are close to the subgradients at the true maximum points.
\begin{lemma}\label{lem:approximate subdifferential}
Suppose $\delta$ is chosen as in Lemma~\ref{lem:separation} and $\Uc$ is bounded: 
$\max_{u\in \Uc}\|u\| = B$.
Then any $z'\in \mathrm{co}$ $\{\cup_{v\in {R^\epsilon_A}} \nabla_u f(u_0,v)\}$ is 
an $(2r\delta B)$-subgradient of $\phi(u_0)$.
\end{lemma}
Now we state our main theorem that if the max step is accurate enough for a large $i$
in terms of $\zeta_i$ (a property of $f$) and $\epsilon_i,\xi_i$ (chosen by a user),
then the algorithm finds the minimum value using a step size $\rho_i \sim 1/i$.
\begin{theorem}\label{thm:main}
Suppose the conditions of Lemmas~\ref{lem:lipschitz},~\ref{lem:separation} and ~\ref{lem:approximate subdifferential} hold, and also
suppose the max step in Alg.~\ref{alg:proposed} is accurate for sufficiently
large $i\geq i_0$ for some $i_0\geq 1$ so that
$\max[ d_H(R(u_i),A_i), d_H(A_i, S(u_i))] \leq \delta_i$ holds 
where $\delta_i \leq \min\left[ 0.5(\zeta_i - \epsilon_i)/l, 
\; 0.5\xi_i/(rB)\right]$ for some non-negative sequence $(\xi_1,\xi_2,...).$ 
If the step size satisfies $\rho_i\geq 0,\forall i$, $\sum_{i=1}^{\infty} \rho_i = \infty$, 
$\sum_{i=1}^\infty \rho_i^2 < \infty$, and $\sum_{i=1}^\infty \rho_i \xi_i <\infty$,
then 
$\min [\phi(u_1),...,\phi(u_i)]$ converges to the minimum value $\phi^\ast$.
\end{theorem}
For $\rho_i$ and $\xi_i$ we can also use $1/i$. 
The $\epsilon_i$ can be any non-negative value.
A large $\epsilon_i$ can make each min step better since the descent direction
in Alg.~\ref{alg:descent direction} uses more $z_i$'s and therefore is more robust. 
The price to pay is that it may take more iterations for the max step to meet the condition
$\delta_i \leq \min\left[ 0.5(\zeta_i - \epsilon_i)/l, \; 0.5\xi_i/(rB)\right]$. 

\subsection{Infinite $R(u)$}

Infinite $R(u)$ is the most challenging case.
We only mention the accuracy of the approximating $R(u)$ with a finite and fixed $A$
as in the grid methods of \citet{dem1971theory,dem1974introduction}.  
\begin{lemma}
For any $\epsilon>0$, one can choose a fixed $A=(v^1,...,v^K)$ such that 
$\phi(u)-\phi_{A}(u) \leq \epsilon$ holds for all $u$. Furthermore,
if $\hat{u}=\arg\min_u \phi_{A}(u)$ is the minimizer of the approximation, 
then $\phi(\uh) - \phi(\us) \leq \epsilon$.
\end{lemma}
If $A$ is dense enough, the solution $\hat{u}$ can be made arbitrarily accurate, 
but the corresponding $K=|A|$ can be too large and has to be limited in practice. 

\subsection{Optional stopping criteria}\label{sec:stopping criteria}

The function $\phi(u)$ is non-smooth and its gradient need not vanish
at the minimum, causing oscillations.
A stopping criterion can help to terminate early.
We can stop at an $\epsilon$-stationary point of $\phi(u)$ 
by checking if $0 \in \partial_\epsilon \phi(u)$ from Lemma~\ref{lem:stationary}. 
Algorithmically, this check is done by solving a LP or a QP problem \cite{demjanov1968algorithms}.
The stopping criterion presented in Alg.~\ref{alg:descent direction} is 
a necessary condition for the approximate stationarity of $\phi(u)$:
\begin{lemma}
Let $\epsilon=\epsilon'+l\delta$ $(\epsilon,\epsilon'\geq 0)$ where $l$ is the
Lipschitz coefficient of $f(u,v)$ in $v$. 
If $u_0$ is an $\epsilon$-stationary point of $\phi(u)$, 
then $u_0$ is an $\epsilon'$-stationary point of $\phi_{A}(u)$.
\end{lemma}
The size $n$ of the QP problem is $|R^\epsilon_A(u)|$ which is small for $\epsilon\ll 1$, 
but it can be costly to solve at every iteration. 
It is therefore more practical to stop after a maximum number of iterations
or by checking the stopping criterion only every so often.

\section{Experiments}\label{sec:experiments}


\subsection{Simple surfaces}\label{sec:exp simple}

\begin{figure*}[thb]
\centering
\includegraphics[width=0.99\linewidth]{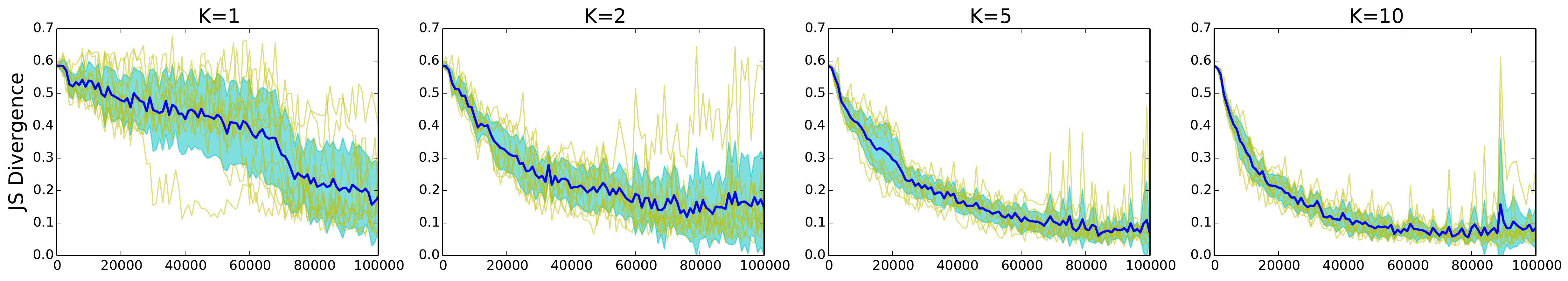}\\
\vspace{0.1in}
\includegraphics[width=0.22\linewidth]{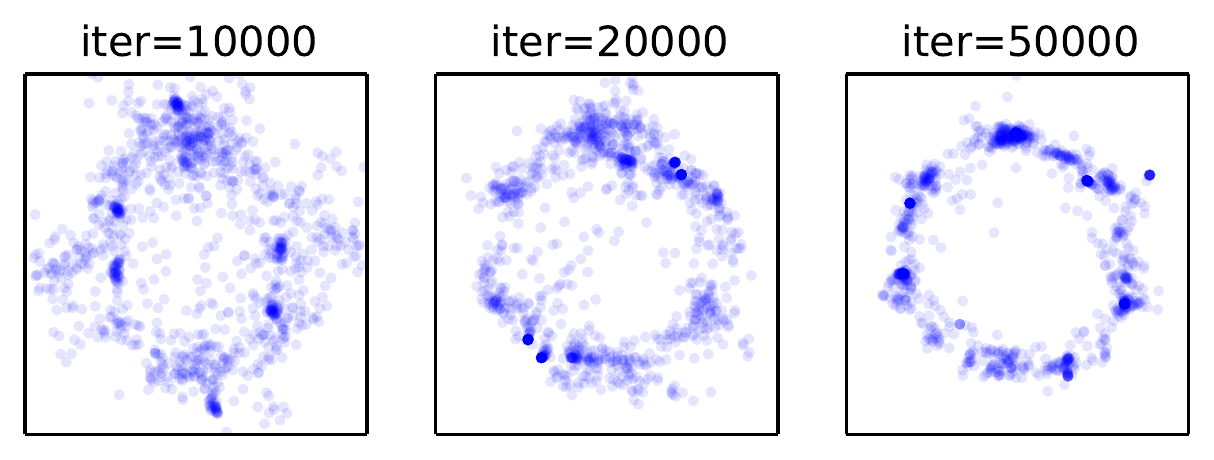}
\hspace{0.1in}
\includegraphics[width=0.22\linewidth]{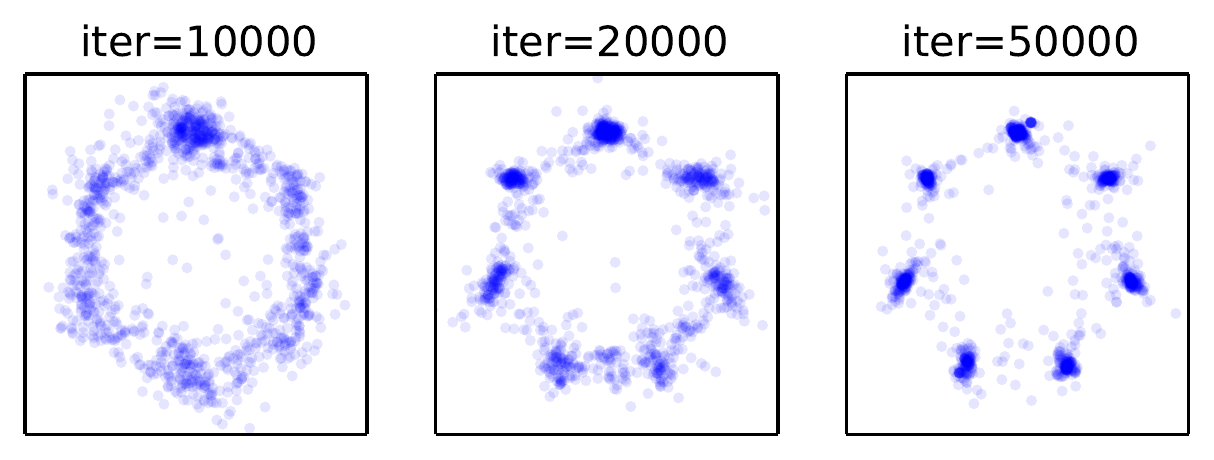}
\hspace{0.1in}	
\includegraphics[width=0.22\linewidth]{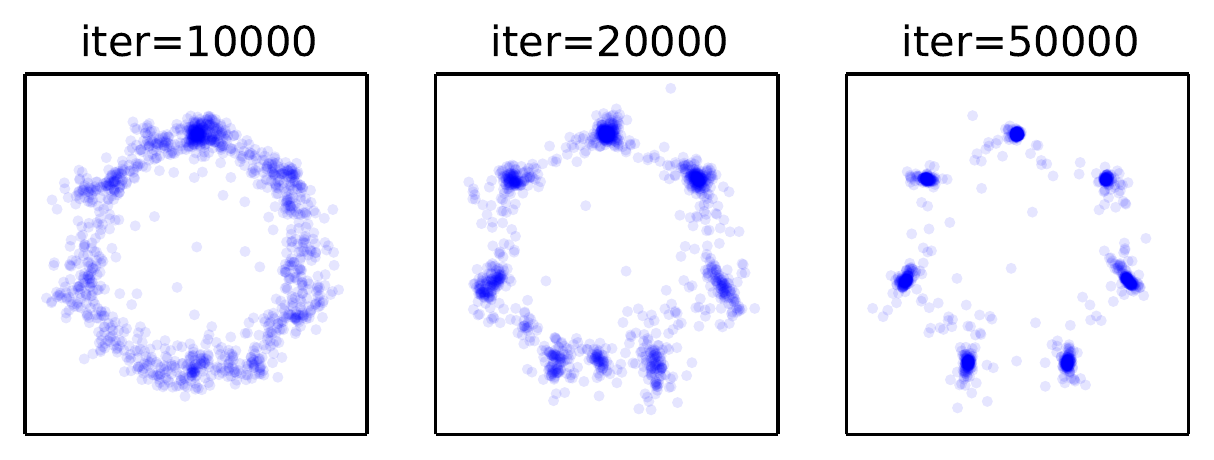}
\hspace{0.1in}	
\includegraphics[width=0.22\linewidth]{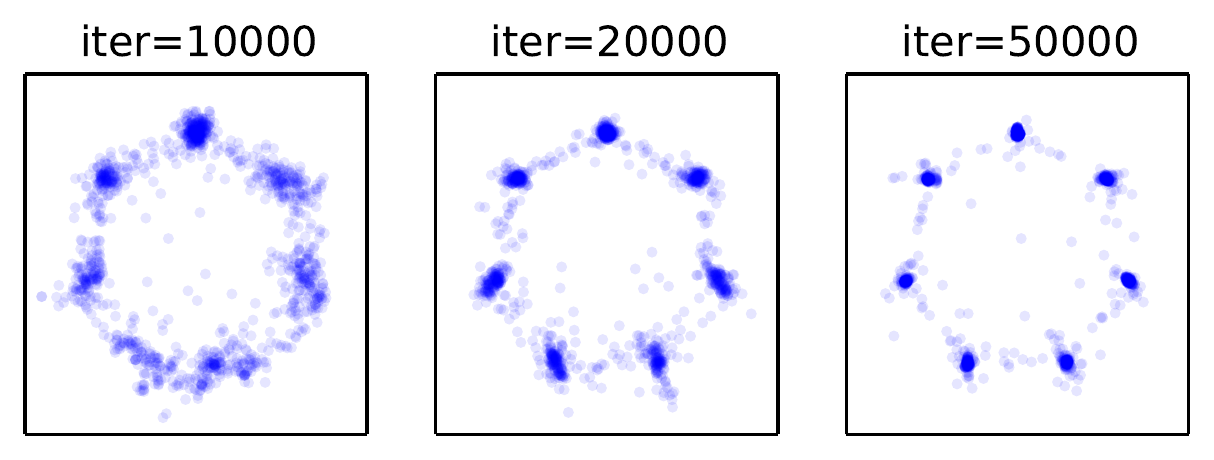}
\caption{Top row: Jensen-Shannon divergence vs iteration for GAN training with MoG. 
The dark blue line is the average divergence and the light blue area is the avg$\pm$std. The light yellow lines are traces of 10 independent trials.
Training is more stable and faster with a larger $K$.
Bottom row: Corresponding samples generated after 10000, 20000, and 50000 iterations.
}
\label{fig:jsd}
\end{figure*}

We test the proposed algorithm to find minimax points of the simple surfaces
in Fig.~\ref{fig:simple}.
We compare Alternating Gradient Descent (Alt-GD), and the proposed $K$-beam algorithm 
with $K=1,2,5,10$. 
Note that for $K=1$, the minimax algorithm is basically the same as Alt-GD.
Since the domain is constrained to $[-0.5,0.5]^2$, we use the projected gradient
at each step with the common learning rate of $\rho_i=\eta_i=0.1/i$. 
In our preliminary tests, the value of $\epsilon_i$ in Alg.~\ref{alg:proposed} 
did not critically affect the results,
and we report the case $\epsilon_i=0$ for all subsequent tests. 
The experiments are repeated for 100 trials with random initial conditions.

Fig.~\ref{fig:simple convergence} shows the convergence of Alt-GD and $K$-beam ($K=1,2,5,10$)
after 200 iterations, 
measured by the distance of the current solution to the closest optimal point
$d(u_i, U^\ast):=\min_{u \in U^\ast} \|u_i - u\|$,
where $U^\ast$ is the set of minimax solutions. 
We plot the average and the confidence level of the 100 trials. 
All methods converge well for surfaces (a) and (b).
The surface (c) is more difficult. Although $(0,0)$ is a saddle point, (i.e., 
$0=f(0,v) \leq f(0,0) \leq f(u,0)=0,\;\forall u,v$), the point $(0,0)$ is unstable
as it has no open neighborhood in which $f$ is a local minimum in $u$ and a local maximum in $v$. 
For non-saddle point problems (d)-(e), one can see that Alt-GD simply cannot
find the true solution, whereas $K$-beam can find the solution if $K$ is large enough. 
For anti-saddle (e), $K=2$ is the smallest number to find the solution
since the local maximum point $|S(u)|$ is at most 2. However, concavity-convexity of $f$
(instead of convexity-concavity) makes optimization difficult and therefore 
$K>2$ helps to recover from bad random initial points and find the solution.

\begin{figure*}[thb]
\centering
\includegraphics[width=0.99\linewidth]{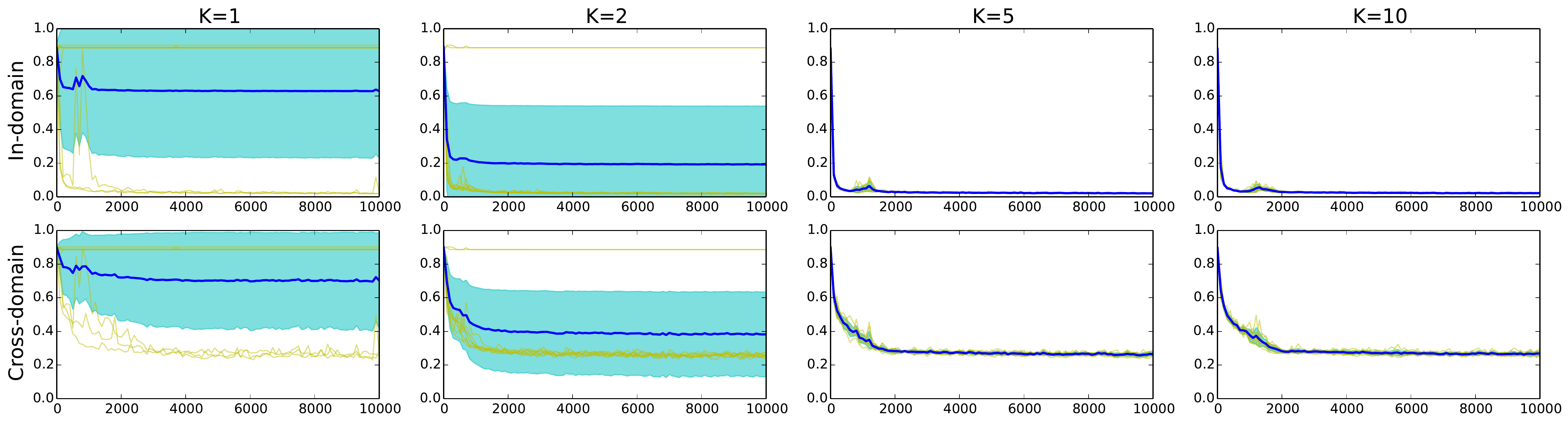}
\caption{Test error vs iteration for MNIST and MNISTM. 
The top/bottom row corresponds to in-domain/cross-domain results, respectively.
The dark blue line is the average error and the light blue area is the avg$\pm$std. 
The light yellow lines are traces of 10 independent trials. 
For $K$=1 or 2, some trials fail to converge at all. With $K$=5 or 10, all trials 
converge to $0.02$ (in-domain) and $0.27$ (cross-domain).
}
\label{fig:da}
\end{figure*}

\subsection{GAN training with MoG}


We train GANs with the proposed algorithm to learn a generative model of 
two-dimensional mixtures of Gaussians (MoGs).
Let $x$ be a sample from the MoG with the density
$p(x) = \frac{1}{7}\sum_{i=0}^6 \mathcal{N}\left((\sin(\pi i/4),\cos(\pi i/4)),\;(0.01)^2I_2\right)$,

and $z$ be a sample from the 256-dimensional Gaussian distribution
$\mathcal{N}(0,I_{256})$.
The optimization problem is
\[
\min_{u}\max_v E\left[\log D(x;v) + \log (1-D(G(z;u);v))\right],
\]
where $G(z;u)$ and $D(x;v)$ are generator and discriminator networks
respectively. 
Both $G$ and $D$ are two-layer tanh networks with 128 hidden units per layer,
trained with Adam optimizer with batch size 128 and the learning rate of $10^{-4}$
for the discriminator and $10^{-3}$ for the generator.  

For evaluation, we measure the Jensen-Shannon divergence 
\[
\mathrm{JSD} = \frac{1}{2}\mathrm{KL}\left(P,\frac{P+Q}{2}\right)+\frac{1}{2}\mathrm{KL}\left(Q,\frac{P+Q}{2}\right)
\]
between the true MoG $P$ and the samples $Q$ from the generator. 
We measure the divergence by discretizing the 2D region into $20\times 20$ bins
and compare the histograms of 64,000 random samples from the generator and 640,000
samples from the MoG.
The top row, Fig.~\ref{fig:jsd}, shows the JSD curves of $K$-beam with $K$=1,2,5,10.
Alt-GD performs nearly the same as $K$=1 and is omitted. 
The results are from 10 trials with random initialization. 
Note first that GAN training is sensitive in that each trial curve is jagged and 
often falls into the ``mode collapsing'' where there is a jump in the curve.
With $K$ increasing, the curve converges faster on average and is more stable
as evidenced by the shrinking variance. 
The bottom row, Fig.~\ref{fig:jsd}, shows the corresponding samples from the generators 
after 10,000, 20,000, and 50,000 iterations from all 10 trials. 
The generated samples are also qualitatively better with $K$ increasing.  

Additionally, we measure the runtime of the algorithms by wall clock on the same
system using a single NVIDIA GTX980 4GB GPU with a single Intel Core i7-2600 CPU.
Even on a single GPU, the runtime per iteration increases only sublinear
in K: relative to the time required for $K$=1, we get $\times$1.07 ($K$=2), $\times$1.63 ($K$=5),
and $\times$2.26 ($K$=10). 
Since the advantages are clear and the incurred time is negligible,
there is a strong motivation to use the proposed method instead of Alt-GD. 

\subsection{Unsupervised domain adaptation}

We perform experiments on unsupervised domain adaptation \citep{ganin2015unsupervised}
which is another example of minimax problems. 
In domain adaption, it is assumed that two data sets belonging to different domains
share the same structure. For examples, MNIST and MNIST-M are both images of handwritten
digits 0--9, but MNIST-M is in color and has random background patches.
Not surprisingly, the classifier trained on MNIST does not perform well with digits from
MNIST-M out of the box.
Unsupervised domain adaption tries to learn a common transformation $G$ of the domains
into another representation/features such that the distributions of the two domains 
are as similar as possible while preserving the digit class information. 
The discriminator $D_1$ tries to predict the domain accurately, and
the target classifier $D_2$ tries to predict the label correctly. 
The optimization problem can be rewritten as $\min_{u=\{u',w\}} \max_v f(u,v)$ with
\[
f(u,v)=-E[D_1(G(x;u');v)] + \lambda\; E[D_2(G(x;u');w)],
\]
which is the weighted difference of the expected risks of the domain
classifier $D_1$ and the digit classifier $D_2$. 
This form of minimax problem has also been proposed earlier
by \citet{hamm2015preserving,hamm2017minimax} to remove sensitive information from data. 
In this experiment, we show domain adaptation results. 
The transformer $G$ is a two-layer ReLU convolutional network that maps the input features
(=images) to an internal representation of dim=2352. 
The discriminator $D_1$ is a single-layer ReLU dense network of 100
hidden units,
and the digit classifier $D_2$ is a two-layer ReLU dense network of 100 hidden units.
All networks are trained with the momentum optimizer with the batch size of 128 
and the learning rate of $10^{-2}$. 
The experiments are repeated for 10 trials with random initialization. 
We use $\lambda=1$.

We performed the task of predicting the class of MNISTM digits, trained using
labeled examples of MNIST and unlabeled examples of MNISTM.
Fig.~\ref{fig:da} shows the classification error of in-domain (top row) and cross-domain
(bottom row) prediction tasks as a function of iterations.
Again we omit the result of Alt-GD as it performs nearly the same as $K$=1. 
With $K$ small, the average error is high for both in-domain and cross-domain tests,
due to failed optimization which can be observed in the traces of the trials.
As $K$ increases, instability disappears and both in-domain and cross-domain errors 
converge to their lowest values.

{\bf Summary and discussions}
\begin{itemize}
\item Experiments with 2D surfaces clearly show that the alternating 
gradient-descent method can fail completely when the minimax points are not
local saddle points, 
while the $K$-beam method can find the true solutions. 
\item For GAN and domain adaptation problems involving nonlinear neural networks,
the $K$-beam and Alt-GD can both find good solutions
if they converge. The key difference is, the $K$-beam \emph{consistently} converges
to a good solution, whereas Alt-GD finds the solution only rarely
(which are the bottom yellow curves for $K$=1 in Fig.~\ref{fig:jsd} and Fig.~\ref{fig:da}.) 
Similar results can be observed in GAN-MNIST experiments in Appendix.
\item The true $K$ value cannot be computed analytically for nontrivial functions.
However, an overestimated $K$ does not hurt the performance theoretically -- it is only redundant. 
One the other hand, an underestimated $K$ can be suboptimal but is
still better than $K$=1. Therefore, in practice, one can choose as large a 
number as allowed by resource limits such as $K$=5 or 10.
\item The $K$-beam method is different from running Alt-GD for $K$-times more
iterations, since the instability of Alt-GD hinders convergence regardless of the 
total number of iterations. 
The $K$-beam method is also different from $K$-parallel independent runs of Alt-GD,
which are basically the figures of $K$=1 in Fig.~\ref{fig:jsd} and Fig.~\ref{fig:da},
but with $K$-times more trials. The variance will be reduced but the average curve will
remain similar.
\end{itemize}


\section{Conclusions} \label{sec:conclusions}

In this paper, we propose the $K$-beam subgradient descent algorithm to solve
continuous minimax problems that appear frequently in machine learning. 
While simple in implementation, the proposed algorithm can significantly improve the
convergence of optimization compared to the alternating gradient descent approach 
as demonstrated by synthetic and real-world examples. 
We analyze the conditions for convergence without assuming concavity or bilinearity,
which we believe is the first result in the literature. 
There are open questions regarding possible relaxations of assumptions used
which are left for future work. 




{\small
\bibliographystyle{icml2018}
\bibliography{icml18_jh}
}




\section*{Appendix}
\appendix

\section{Simple surfaces}

\begin{figure*}[thb]
\begin{subfigure}{.33\linewidth}
	\centering
	\captionsetup{justification=centering}
	\includegraphics[width=0.66\linewidth]{demo_surfaces1_points.pdf}
	\includegraphics[width=0.32\linewidth]{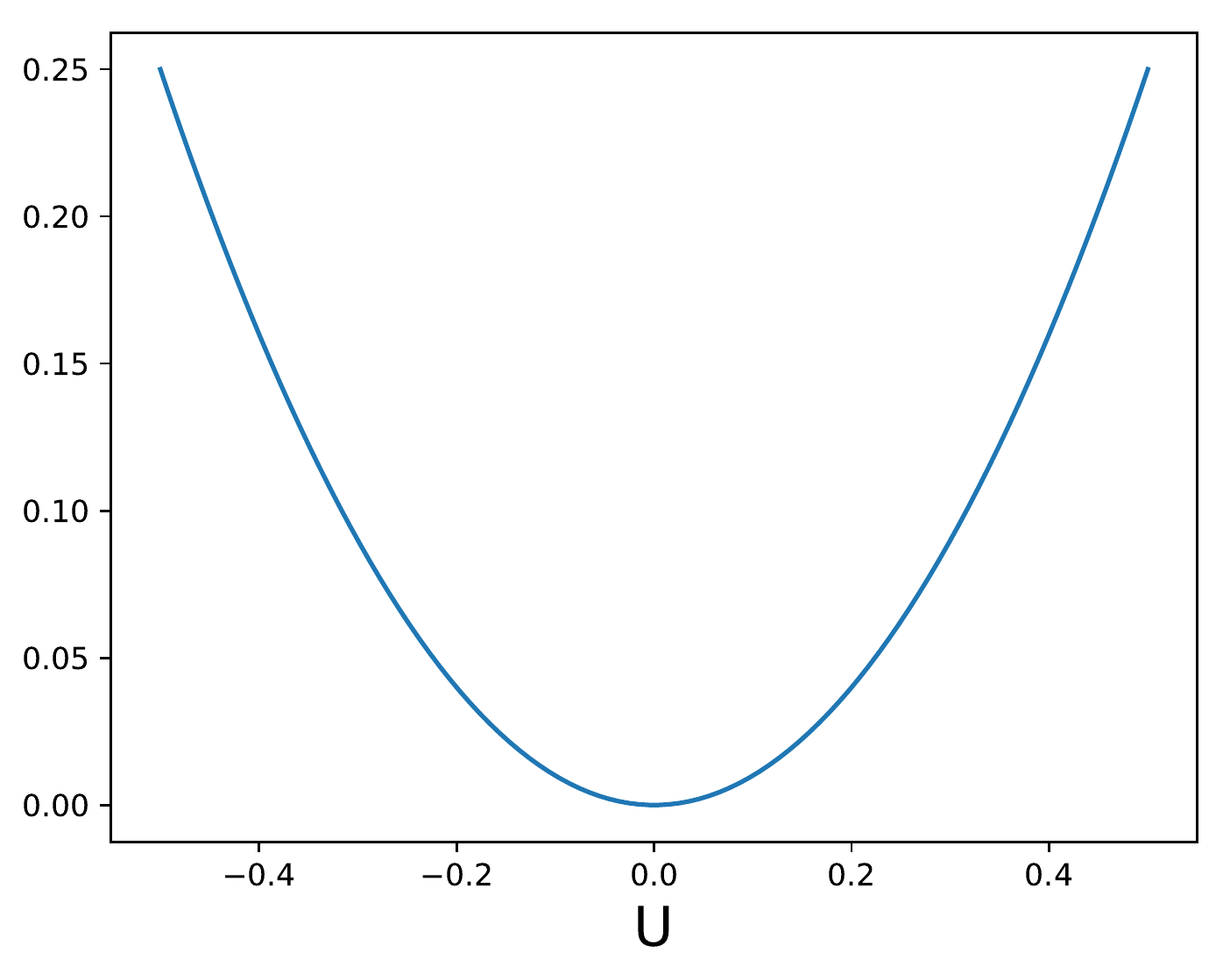}
	\caption{
	Saddle ($u^2-v^2$)\\
	\small
	\begin{tabular}{|c|c|}
	\hline
	critical pts & $\{(0,0)\}$ \\
	saddle pts & $\{(0,0)\}$ \\
	minimax pts & $\{(0,0)\}$\\
	\hline
	\end{tabular} 
	}
\end{subfigure}
\begin{subfigure}{.33\linewidth}
	\centering
	\captionsetup{justification=centering}
	\includegraphics[width=0.66\linewidth]{demo_surfaces2_points.pdf}
	\includegraphics[width=0.32\linewidth]{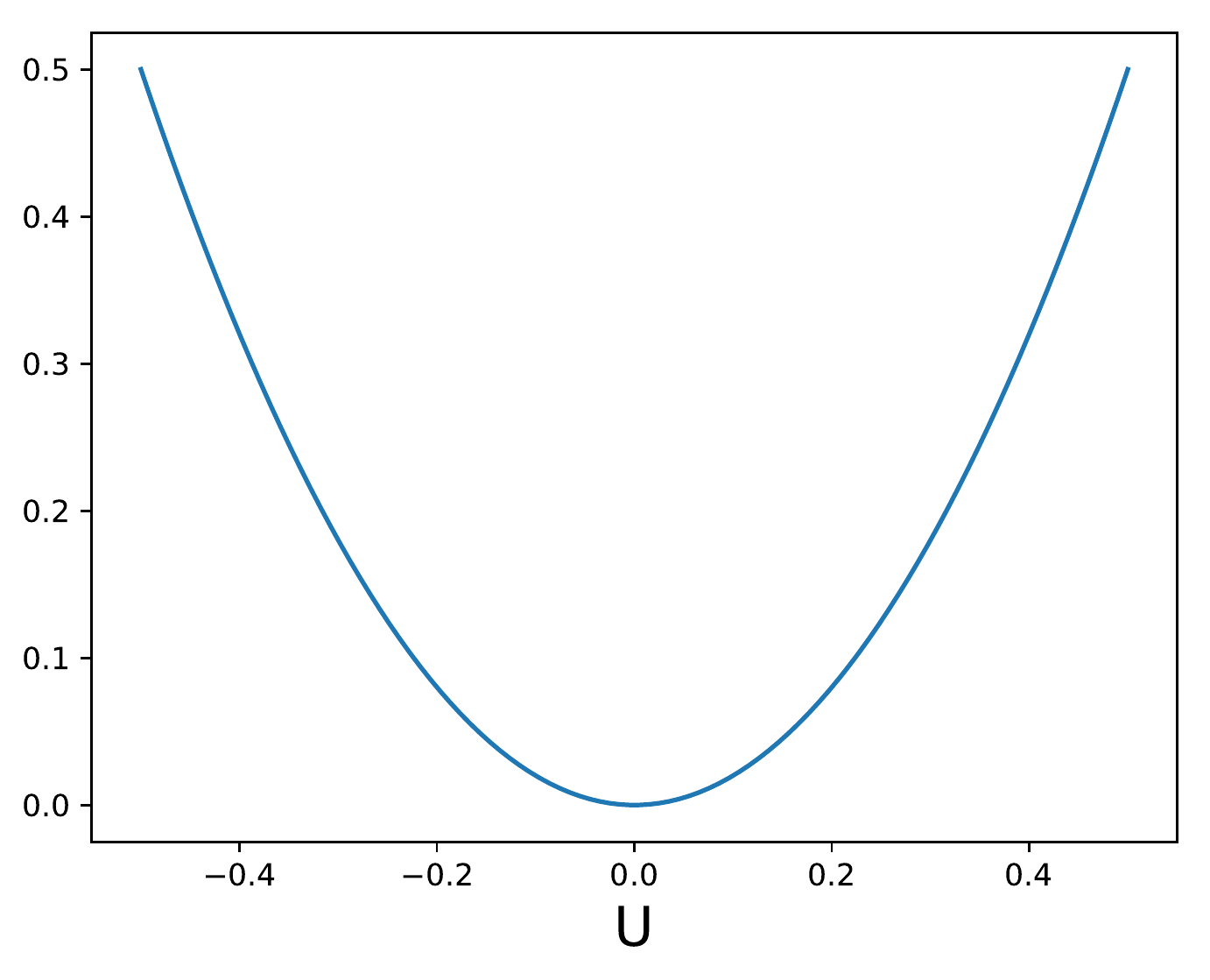}
	\caption{
	Rotated saddle ($u^2-v^2+2uv$)
	\small
	\begin{tabular}{|c|c|}
	\hline
	critical pts & $\{(0,0)\}$ \\
	saddle pts & $\{(0,0)\}$ \\
	minimax pts & $\{(0,0)\}$\\
	\hline
	\end{tabular} 
	}
\end{subfigure}
\begin{subfigure}{.33\linewidth}
	\centering
	\captionsetup{justification=centering}
	\includegraphics[width=0.66\linewidth]{demo_surfaces8_points.pdf}
	\includegraphics[width=0.32\linewidth]{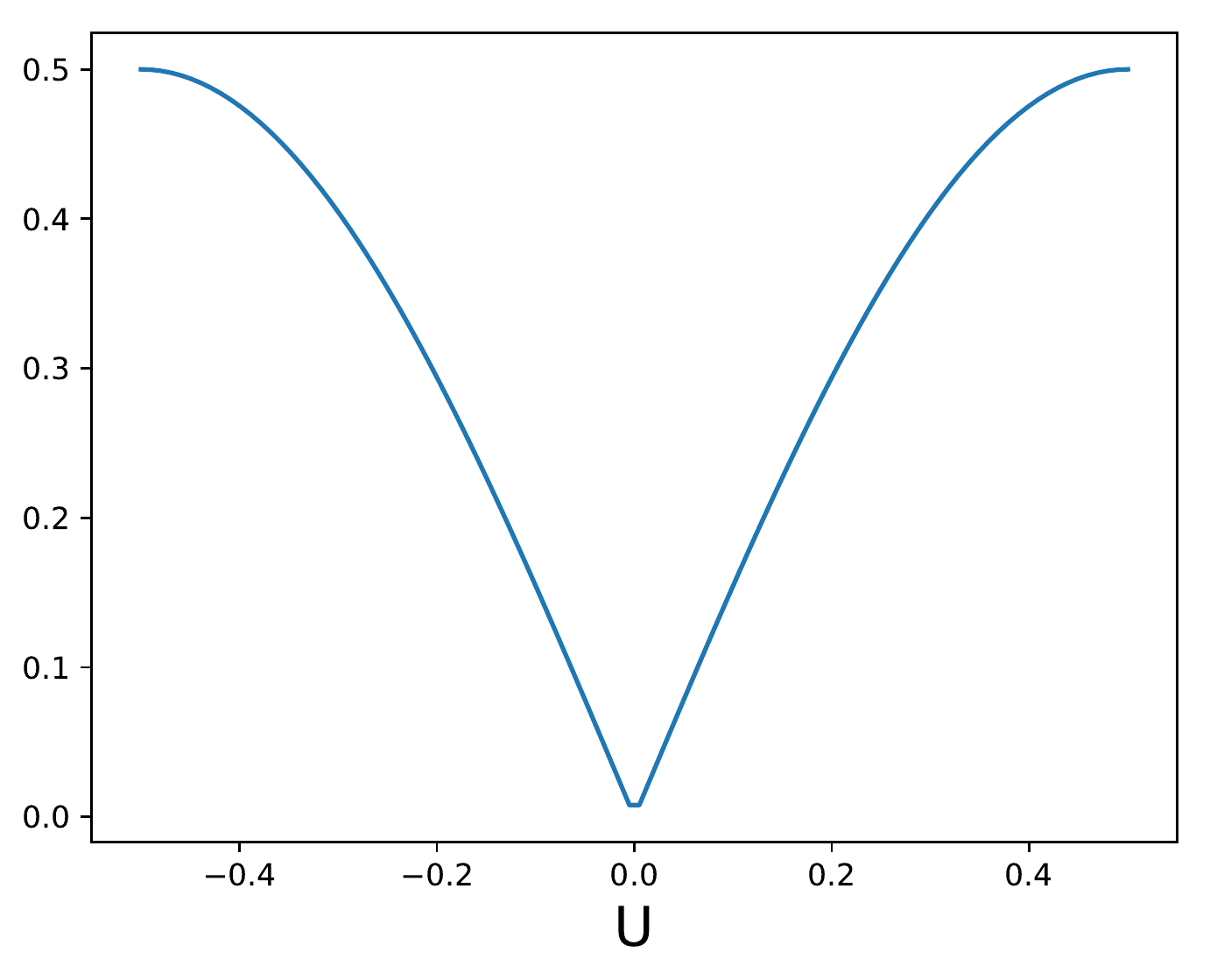}
	\caption{
	Seesaw ($-v\sin(\pi u)$)	
	\small
	\begin{tabular}{|c|c|}
	\hline
	critical pts & $\{(0,0)\}$ \\
	saddle pts & $\{(0,0)\}$ \\
	minimax pts & $\{(0,v)|v \in [-0.5,0.5])\}$\\
	\hline
	\end{tabular} 
	}
\end{subfigure}
\begin{subfigure}{.33\linewidth}
	\centering
	\captionsetup{justification=centering}
	\includegraphics[width=0.66\linewidth]{demo_surfaces32_points.pdf}
	\includegraphics[width=0.32\linewidth]{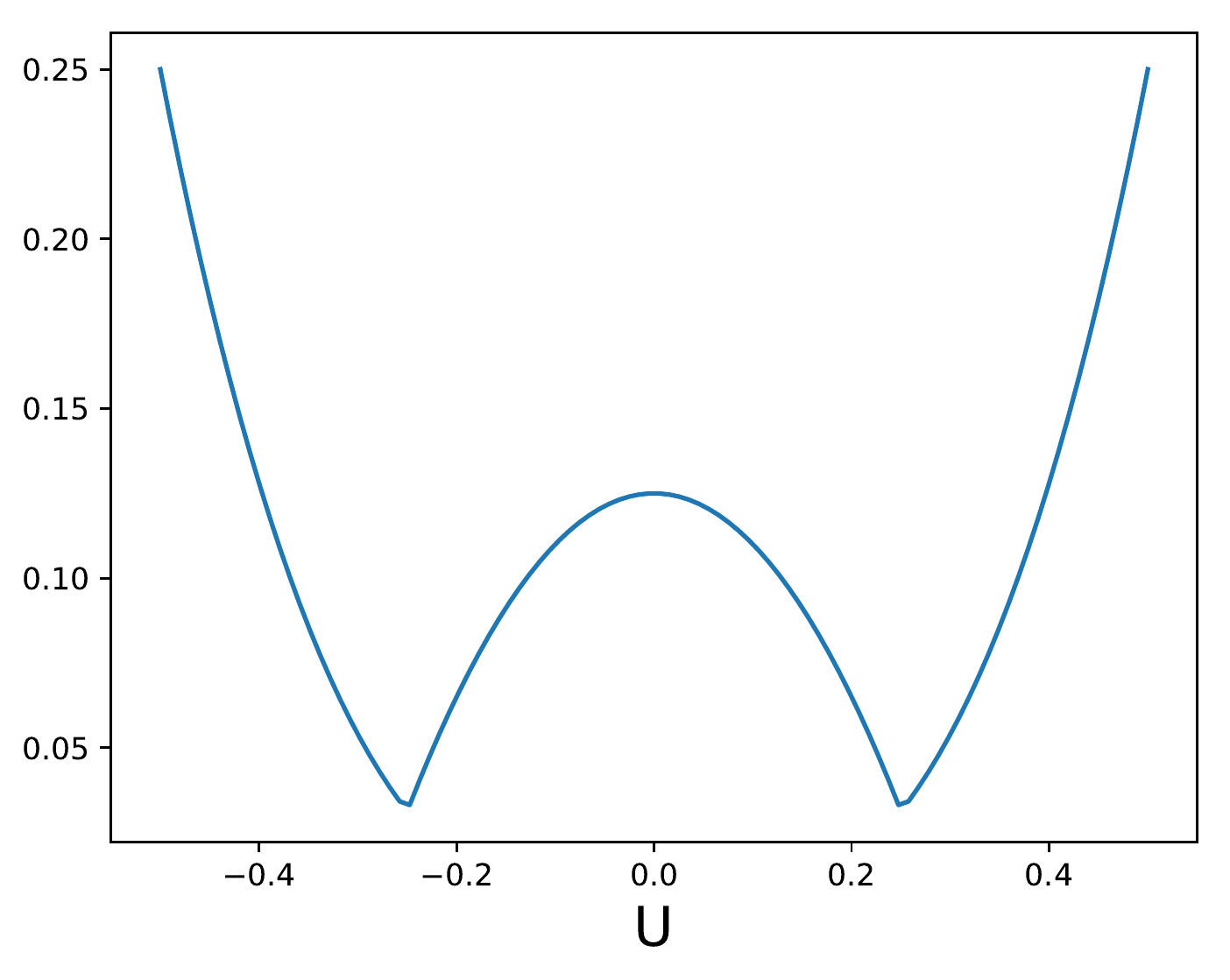}
	\caption{
	Monkey saddle ($v^3-3vu^2$)
	\small
	\begin{tabular}{|c|c|}
	\hline
	critical pts & $\{(0,0)\}$ \\
	saddle pts & $\{\}$ \\
	minimax pts & $\{(\pm0.25,-0.25),(\pm0.25,0.5)\}$\\
	\hline
	\end{tabular} 
	}
\end{subfigure}
\begin{subfigure}{.33\linewidth}
	\centering
	\captionsetup{justification=centering}
	\includegraphics[width=0.66\linewidth]{demo_surfaces12_points.pdf}
	\includegraphics[width=0.32\linewidth]{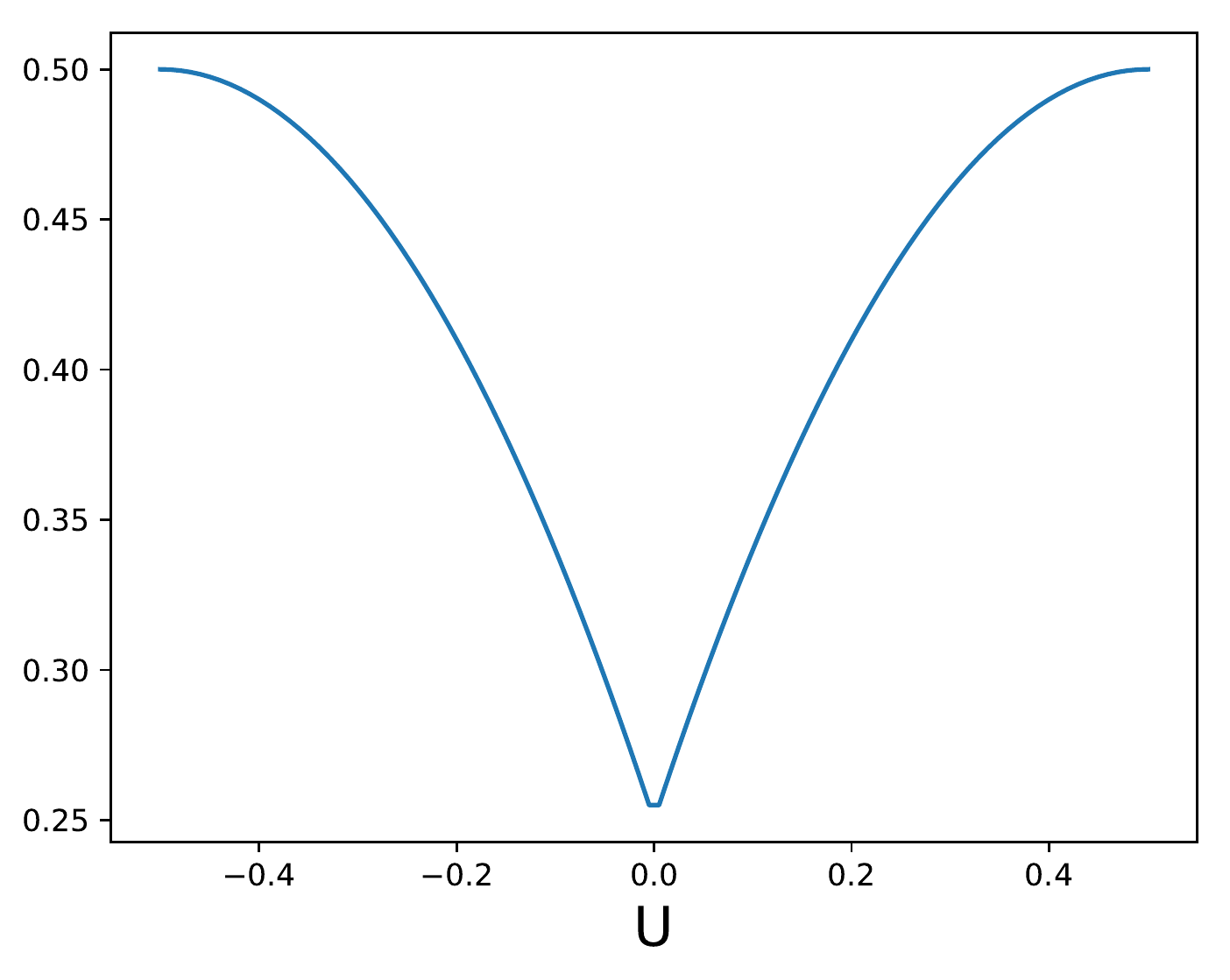}
	\caption{
	Anti-saddle ($-u^2+v^2+2uv$)
	\small
	\begin{tabular}{|c|c|}
	\hline
	critical pts & $\{(0,0)\}$ \\
	saddle pts & $\{\}$ \\
	minimax pts & $\{(0,\pm0.5)\}$\\
	\hline
	\end{tabular} 
	}
\end{subfigure}
\begin{subfigure}{.33\linewidth}
	\centering
	\captionsetup{justification=centering}
	\includegraphics[width=0.66\linewidth]{demo_surfaces9_points.pdf}
	\includegraphics[width=0.32\linewidth]{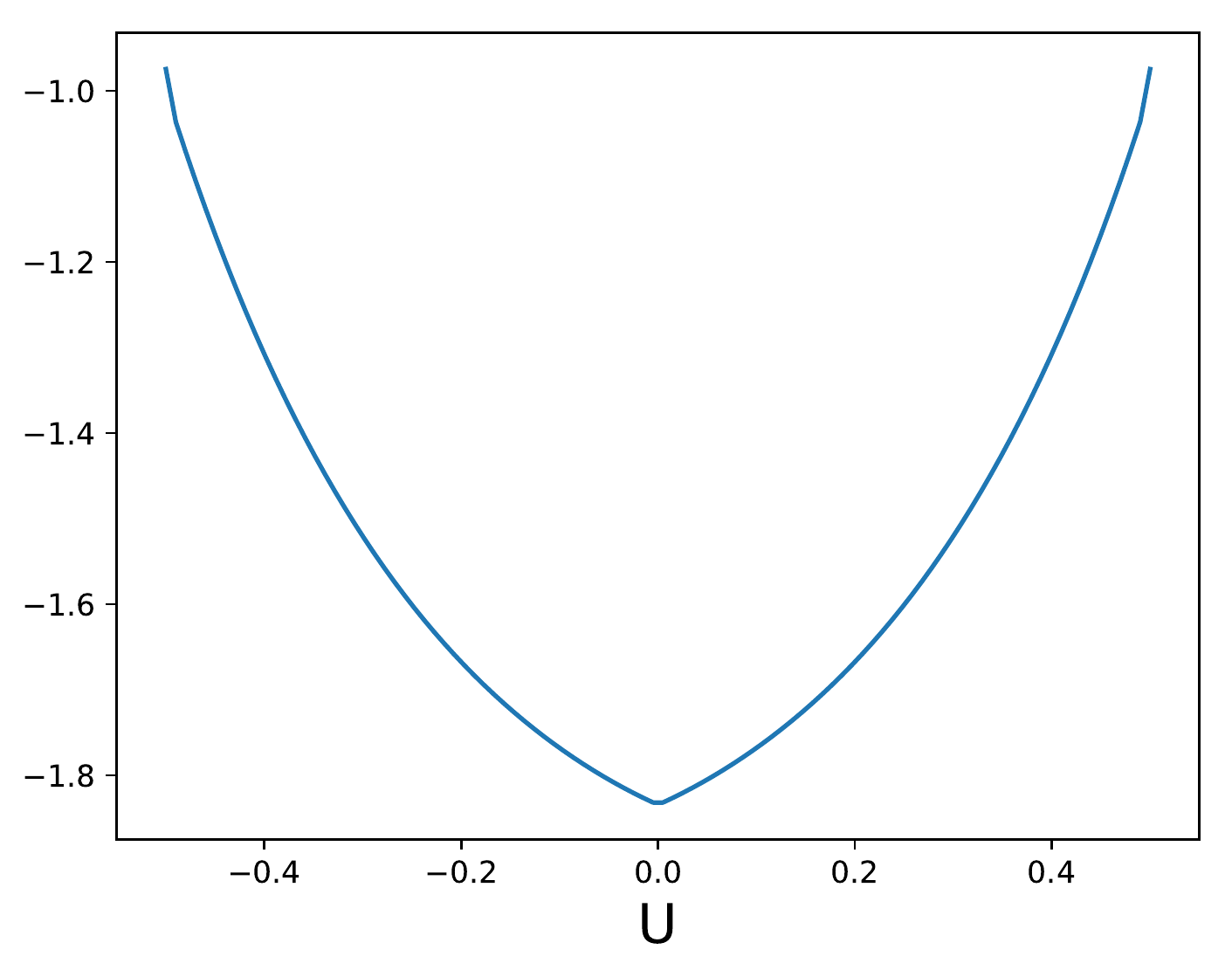}
	\caption{
	Weapons ($ e^{-10(u+.5)e^{-(v+.5)}}+e^{-10(.5-u)e^{v-.5}}$)
	\small
	\begin{tabular}{|c|c|}
	\hline
	critical pts & $\{(0,0)\}$ \\
	saddle pts & $\{\}$ \\
	minimax pts & $\{(0,\pm0.5)\}$\\
	\hline
	\end{tabular} 
	}
	
\end{subfigure}
\caption{Examples of saddle point (top row) and non-saddle point (bottom row)
problems. The smaller inset after each surface is the max value function $\phi(u)=\max_v f(u,v)$.
}
\label{fig:simple2}
\end{figure*}

Fig.1 shows the six surfaces $f(u,v)$ and the maximum value function 
$\phi(u)=\max_{v\in \Vc} f(u,v)$.
From $\phi(u)$ one can check the minima $\arg\min_u \phi(u)$ are:\\
(a) $u=0$, (b) $u=0$, (c) $u=0$, (d) $u=\pm 0.25$, (e) $u=0$, and (f) $u=0$.\\
The corresponding maxima $R(u)=\arg\max_{v\in \Vc} f(u,v)$ at the minimum are:\\
(a) $R(0)=\{0\}$, (b) $R(0)=\{0\}$, (c) $R(0)=[-0.5,0.5]$, 
(d) $R(\pm0.25)=\{-0.25,0.5\}$, (e) $R(0)=\{-0.5,0.5\}$, and (f) $R(0)=\{-0.5,0.5\}$.

Furthermore, $R(\Uc)$ for the whole domain is:\\
(a) $R(\Uc)=\{0\}$, (b) $R(\Uc)=[-0.5,0.5]$,
(c) $R(\Uc)=\{-0.5,0.5\}$ except for $R(0)=[-0.5,0.5]$, (d) $R(\Uc)=[-0.5,-0.25] \cup \{0.5\}$,
(e) $R(\Uc)=\{-0.5,0.5\}$, and (f) $R(\Uc)=\{-0.5,0.5\}$.
These can be verified by solving the minimax problems in closed form.

Note that the origin $(0,0)$ is a critical point for all surfaces.
It is also a global saddle point and minimax point for surfaces (a)-(c),
but is neither a saddle nor a minimax point for surfaces (d)-(f).

\section{Proofs}

\setcounter{theorem}{0}

\begin{lemma}[Corollary 4.3.2, Theorem 4.4.2, \cite{hiriart2001fundamentals}]
Suppose $f(u,v)$ is convex in $u$ for each $v \in A$. 
Then $\partial \phi_A(u)=\mathrm{co}\{\cup_{v \in A} \nabla_u f(u,v)\}$.
Similarly, suppose $f(u,v)$ is convex in $u$ for each $v \in \Vc$. Then $\partial \phi(u)=\mathrm{co}\{\cup_{v \in \Vc} \nabla_u f(u,v)\}$.
\end{lemma}

\begin{lemma}[Chap 3.6, \cite{dem1974introduction}]\label{lem:stationary}
A point $u$ is an $\epsilon$-stationary point of $\phi_A(u)$
if and only if $0 \in \mathrm{co}\{\cup_{v \in R^\epsilon_A(u)} \nabla_u f(u,v)\}$.
\end{lemma}

\begin{lemma}\label{lem:zero dist}
Suppose $R(u)$ is finite at $u$.
If $d_H(R(u),A)=0$, then $R(u) = R_{A}(u)$ and therefore 
$\partial \phi(u) = \partial \phi_{A}(u)$.
\end{lemma}
\begin{proof}
Since $A \subseteq \Vc$ , $\max_{v\in\Vc}f(u,v)=\max_{v\in R(u)}$ $ f(u,v)\geq \max_{v \in A} f(u,v)$. 
By $d_H(R(u),A)=0$, we have $R(u) \subseteq A$ and therefore for each $v \in R(u)$,
$f(u,v)=\max_{v\in \Vc} f(u,v) = \max_{v\in A}f(u,v)$, so $v \in R_A(u)$.
Conversely, if $v \in R_{A}(u)$ then $f(u,v)=\max_{v \in A}f(u,v)=\max_{v \in \Vc}f(u,v)$, so $v \in R(u)$. 
The remainder of the theorem follows from the definition of subdifferentials.
\end{proof}

Fig.~\ref{fig:auxil} explains several symbols used in the following lemmas. 
\begin{figure}[thb]
\centering
\includegraphics[width=0.7\linewidth]{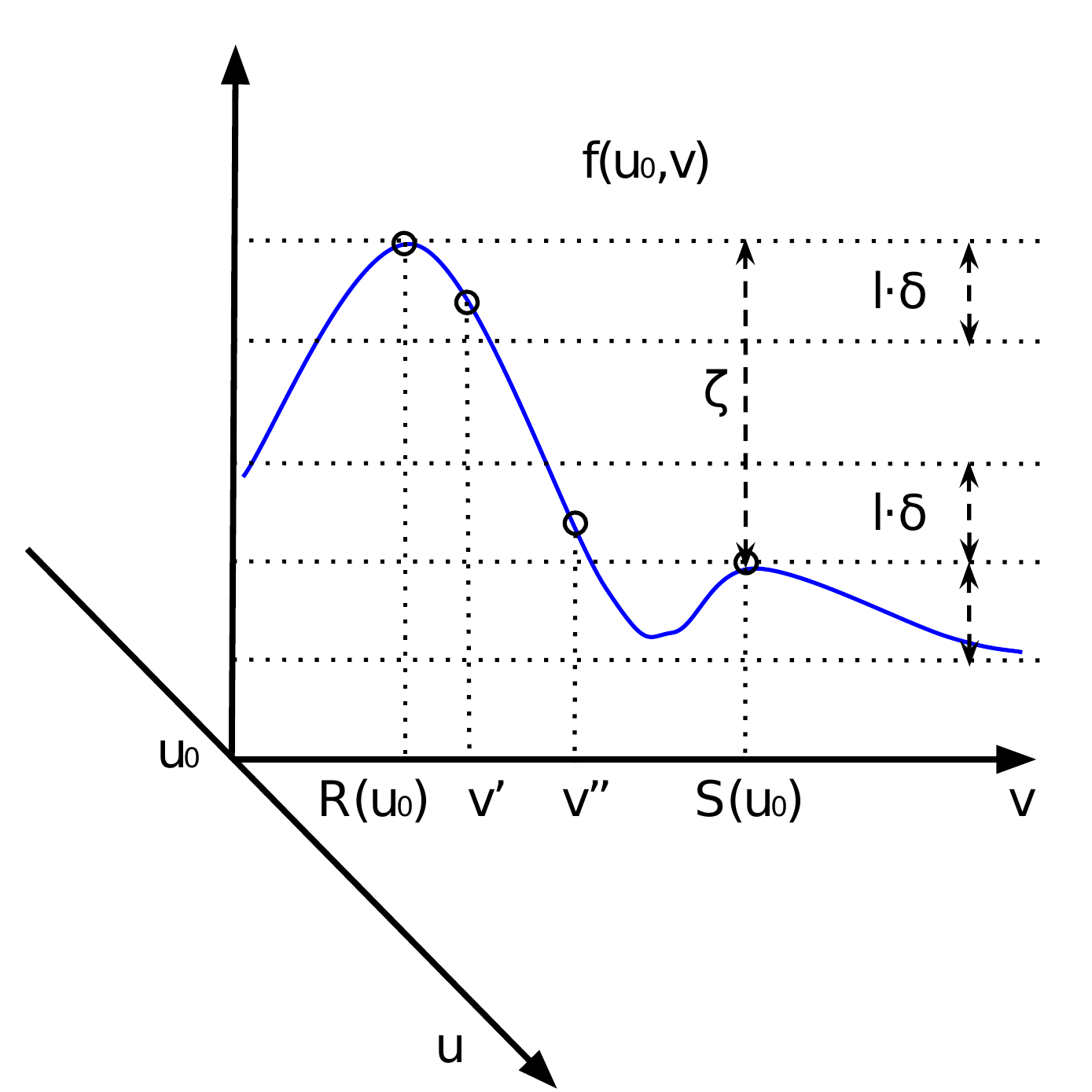}
\caption{Consider a slice of $f(u,v)$ at $u=u_0$.
$\zeta$: smallest gap between the $f$ values of global maxima $R(u)$ and 
non-global maxima $S(u)\setminus R(u)$. 
$v'$ is no farther than $\delta$ to a point in $R(u_0)$ and 
$v''$ is no farther than $\delta$ to a point in $S(u_0)\setminus R(u_0)$. 
By choosing $\epsilon< \zeta-2l\delta$, we have $v' \in R^\epsilon_A(u_0)$ and
$v'' \notin R^\epsilon_A(u_0)$. See Lemma~\ref{lem:separation}.
}
\label{fig:auxil}
\end{figure}

\begin{lemma}\label{lem:lipschitz}
If $d_H(R(u),A)\leq \delta$, then for each $v \in R(u)$ there is one or more
$v' \in A$ such that $\phi(u) - f(u,v') \leq l \delta$ and 
$\|\nabla_u f(u,v) - \nabla_u f(u,v')\| \leq r \delta$. 
\end{lemma}
The proof follows directly from the Lipschitz assumptions. 

\begin{lemma}\label{lem:separation}
Assume $R(u)$ and $S(u)$ are both finite at $u$.
Let $\zeta = \phi(u) - \max_{v \in S(u)\setminus R(u)} f(u,v)$ be the
smallest gap between the global and the non-global maximum values at $u$. 
If all local maxima are global maxima, then set $\zeta=\infty$.
If $d_H(R(u),A)\leq \delta$ and $d_H(A,S(u))\leq \delta$ where
$\delta < 0.5(\zeta-\epsilon)/l$, 
then for each $v' \in R^\epsilon_A(u)$, there is
$v \in R(u)$ such that $\|v - v'\|\leq \delta$.
\end{lemma}
\begin{proof}
Let any $v' \in A$ be $\delta$-close to a global maximum, then 
$f(u,v') \geq \phi(u) - l \delta$.
Similarly, let any $v'' \in A$ be $\delta$-close to a non-global maximum,
then $f(u,v'') \leq \phi(u) -(\zeta - l\delta)$.
Consequently, 
$f(u,v') \geq f(u,v'') + \zeta - 2l\delta > f(u,v'') + \epsilon$, i.e.,
any $f(u,v')$ and $f(u,v'')$ are separated by at least $\epsilon$.
Therefore, each $v'$ satisfies $v' \in R^\epsilon_A= \{v\in A\;|\;\phi_A(u)-f(u,v)\leq \epsilon\}$
but no $v''$ satisfies $v'' \in R^\epsilon_A$.
\end{proof}

\begin{lemma}\label{lem:approximate subdifferential}
Suppose $\delta$ is chosen as in Lemma~\ref{lem:separation} and $\Uc$ is bounded
($\forall u\in\Uc,\;\|u\| = B < \infty$.)
Then any $z'\in \mathrm{co}\{\cup_{v\in {R^\epsilon_A}} \nabla_u f(u_0,v)\}$ is 
an $(2r\delta B)$-subgradient of $\phi(u_0)$.
\end{lemma}
\begin{proof}
From Lemmas~\ref{lem:lipschitz} and ~\ref{lem:separation},
for each $(v^k)'\in R^\epsilon_A$, there is $v^k \in R(u_0)$ such that 
$\|\nabla_u f(u_0,v^k) - \nabla_u f(u_0,(v^k)')\| \leq r\delta.$
Let $z_k = \nabla_u f(u_0,v^k)$ and $z'_k = \nabla_u f(u_0,(v^k)')$.
Then, for all $k=1,...,|R^\epsilon_A|$ and for all $u$, 
\begin{eqnarray*}
&&\phi(u)-\phi(u_0) - \langle z_k',\;u-u_0\rangle\\
&=& \phi(u)-\phi(u_0) - \langle z_k+z'_k-z_k,\;u-u_0\rangle\\
&\geq& - \langle z'_k-z_k,\;u-u_0\rangle\\
&\geq& -\|z'_k-z_k\|\|u-u_0\|\\
&\geq& -r\delta \|u-u_0\|\geq -2r\delta B.
\end{eqnarray*}
By taking any convex combination of $\sum_{k=1}^n a_k (\cdot)$ on both sides, 
we have 
\[
\phi(u)-\phi(u_0) - \langle \sum_{k=1}^n  a_k z'_k,\;u-u_0\rangle \geq -2r\delta B,
\]
and therefore any $z'\in \mathrm{co}\{\cup_{v\in {R^\epsilon_A}} \nabla_u f(u_0,v)\}$
is a $(2r\delta B)$-subgradient of $\phi(u_0)$
\end{proof}

\begin{theorem}\label{thm:main}
Suppose the conditions of Lemmas~\ref{lem:lipschitz},~\ref{lem:separation} and ~\ref{lem:approximate subdifferential} hold, and also
suppose the max step in Alg.2 is accurate for sufficiently
large $i\geq i_0$ for some $i_0\geq 1$ so that
$\max[ d_H(R(u_i),A_i), d_H(A_i, S(u_i))] \leq \delta_i$ holds 
where $\delta_i \leq \min\left[ 0.5(\zeta_i - \epsilon_i)/l, 
\; 0.5\xi_i/(rB)\right]$ for some non-negative sequence $(\xi_1,\xi_2,...).$ 
If the step size satisfies $\rho_i\geq 0,\forall i$, $\sum_{i=1}^{\infty} \rho_i = \infty$, 
$\sum_{i=1}^\infty \rho_i^2 < \infty$, and $\sum_{i=1}^\infty \rho_i \xi_i <\infty$,
then 
$\min [\phi(u_1),...,\phi(u_i)]$ converges to the minimum value $\phi^\ast$.
\end{theorem}
Note that a stronger result such as $\lim\inf_{i\to \infty} \phi(u_i) = \phi^\ast$
is possible (see, e.g., \cite{correa1993convergence}), but we give a simpler proof
similar to \cite{boyd2003subgradient} which assumes $\|\nabla_u f(u,v)\| \leq L$ for some $L>0$. 
\begin{proof}
We combine previous lemmas with the standard proof of the $\epsilon$-subgradient descent
method.
Let $u_{i+1} = u_i - \rho_i g_i$. Then,
\begin{eqnarray*}
&&\|u_{i+1}-\us\|^2\\
 &=& \|u_i-\us\|^2 + \rho_i^2 \|g_i\|^2 + 2\rho_i \langle g_i,\; \us-u_i\rangle\\
&\leq& \|u_i-\us\|^2 + \rho_i^2 \|g_i\|^2+2\rho_i(\phi(\us)-\phi(u_i)+\xi_i)
\end{eqnarray*}
from the definition of $\partial_\xi \phi(u)$.
Taking $\sum_{i=1}^N (\cdot)$ on both sides gives us
\begin{eqnarray*}
\|u_{N+1}-\us\|^2 &\leq& \|u_1 - \us\|^2 +\sum_{i=1}^N \rho_i^2 \|g_i\|^2 
\\ 
&& +2 \sum_{i=1}^N \rho_i(\phi(\us)-\phi(u_i)+\xi_i),
\end{eqnarray*}
or equivalently,
\[
2 \sum_{i=1}^N ( \rho_i(\phi(u_i)-\phi(\us)-\xi_i)  \leq  \|u_1 - \us\|^2 +\sum_{i=1}^N \rho_i^2 \|g_i\|^2.
\] 
If we define $\underline{\phi}(u_i):=\min [\phi(u_1), ... , \phi(u_i)]$, then
$\sum_{i=1}^N \rho_i (\phi(u_i)-\phi^\ast) \geq (\sum_{i=1}^N \rho_i) (\underline{\phi}(u_i) - \phi^\ast)$.
Combining the two inequalities, we have
\begin{eqnarray*}
0 &\leq& \underline{\phi}(u_i)-\phi^\ast \leq 
\frac{\sum_{i=1}^N \rho_i (\phi(u_i)-\phi^\ast)}{\sum_{i=1}^N \rho_i} \\
&\leq&
\frac{\|u_1 - \us\|^2 +\sum_{i=1}^N \rho_i^2 \|g_i\|^2 + 2 \sum_{i=1}^N \rho_i \xi_i}{2\sum_{i=1}^N \rho_i}  \\
&\leq&
\frac{\|u_1 - \us\|^2 +\sum_{i=1}^N \rho_i^2 L^2 +2\sum_{i=1}^N \rho_i \xi_i}{2\sum_{i=1}^N \rho_i}.
\end{eqnarray*}
With $\sum_{i=1}^{\infty} \rho_i = \infty$,  
$\sum_{i=1}^\infty \rho_i^2 < \infty$, and $\sum_{i=1}^\infty \rho_i \xi_i <\infty$, 
we get $\underline{\phi}(u_i) \to \phi^\ast$. 
\end{proof}
\begin{lemma}
For any $\epsilon>0$, one can choose a fixed $A=(v^1,...,v^k)$ such that 
$\phi(u)-\phi_{A}(u) \leq \epsilon$ holds for all $u$. Furthermore,
if $\hat{u}=\arg\min_u \phi_{A}(u)$ is the minimizer of the approximation, 
then $\phi(\uh) - \phi(\us) \leq \epsilon$.
\end{lemma}
\begin{proof}
Since $\Vc$ is compact and $f$ is continuous, we can find a finite grid $A$
such as a uniform $\epsilon/l$-grid for $l$-Lipschitz $f$ so that
$\phi(u)-\phi_{A}(u) \leq \epsilon$ for all $u$.
Furthermore, we have
\begin{eqnarray*}
\phi(\uh) - \phi(\us) 
&=& \phi(\uh) - \phi_{A}(\uh)  + \phi_{A}(\uh) - \phi(\us)\\
&\leq& \phi(\uh) - \phi_{A}(\uh)  + \phi_{A}(\us) - \phi(\us)\\
&\leq& \phi(\uh) - \phi_{A}(\uh) \leq \epsilon,
\end{eqnarray*}
since $\phi_{A}(u)=\max_{v \in A}f(u,v) \leq \max_{v \in \Vc}f(u,v)=\phi(u)$
for all $u$.
\end{proof}

\begin{lemma}
Let $\epsilon=\epsilon'+l\delta$ $(\epsilon,\epsilon'\geq 0)$ where $l$ is the
Lipschitz coefficient of $f(u,v)$ in $v$. 
If $u_0$ is an $\epsilon$-stationary point of $\phi(u)$, 
then $u_0$ is also an $\epsilon'$-stationary point of $\phi_{A}(u)$.
\end{lemma}
\begin{proof}
At the $\epsilon'$-stationary point of $\phi_A$, we have
$\max_{v \in R^{\epsilon'}_A} \langle \nabla_u f(u,v),\;g\rangle \geq 0$ for all $g$
by definition.
Since $R^{\epsilon}(u)=R^{\epsilon'+l\delta}(u) \supseteq R^{\epsilon'}_A(u)$, 
we have $\max_{v \in R^{\epsilon}} \langle \nabla_u f(u,v),\;g\rangle \geq 
\max_{v \in R^{\epsilon'}_A} \langle \nabla_u f(u,v),\;g\rangle \geq 0$ for all $g$.
\end{proof}

\section{GAN training for MNIST}

\begin{figure}[thb]
\centering
\begin{subfigure}{.49\linewidth}
	\centering
	\includegraphics[width=0.99\linewidth]{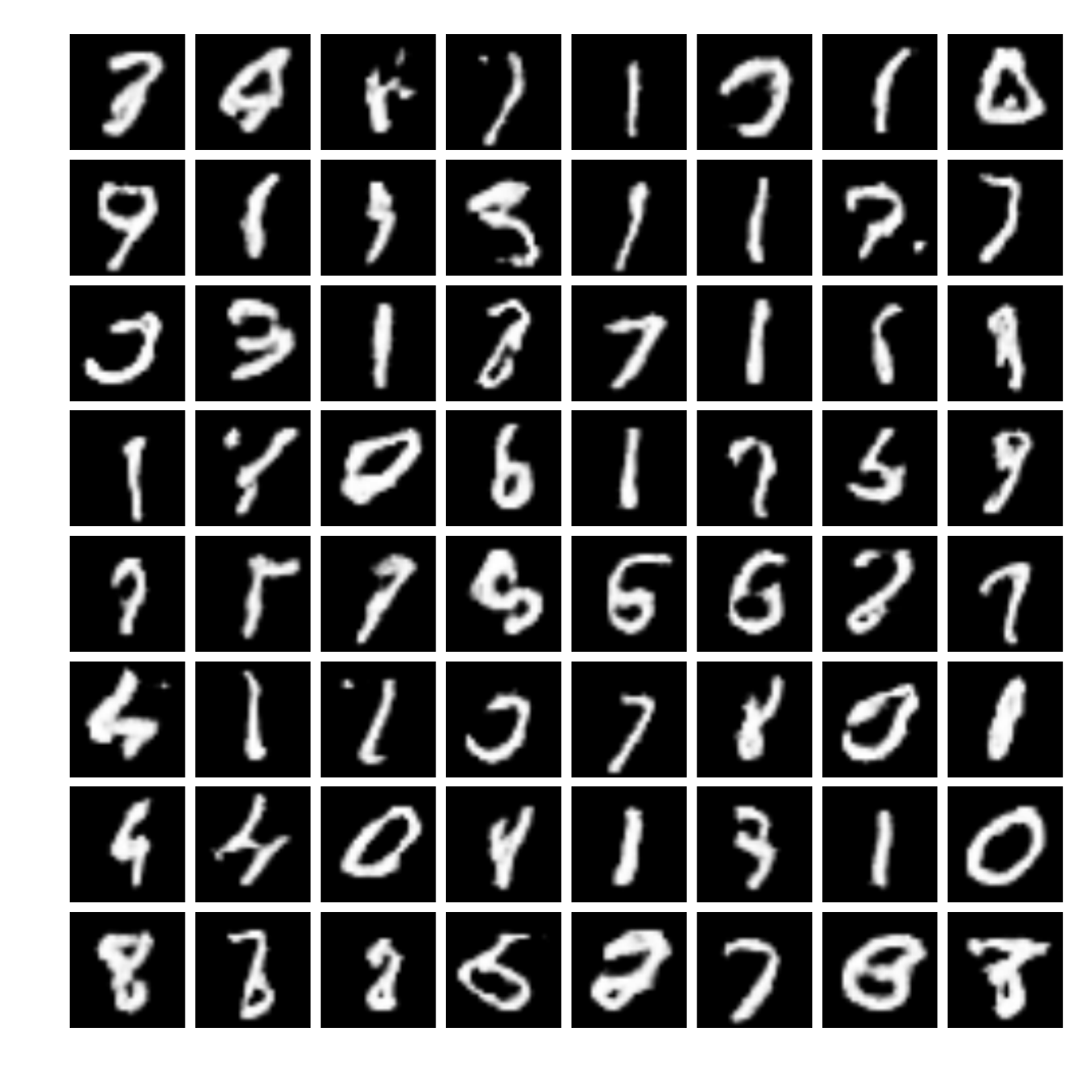}
	\caption{$K=1$}
\end{subfigure}
\begin{subfigure}{.49\linewidth}
	\centering
	\includegraphics[width=0.99\linewidth]{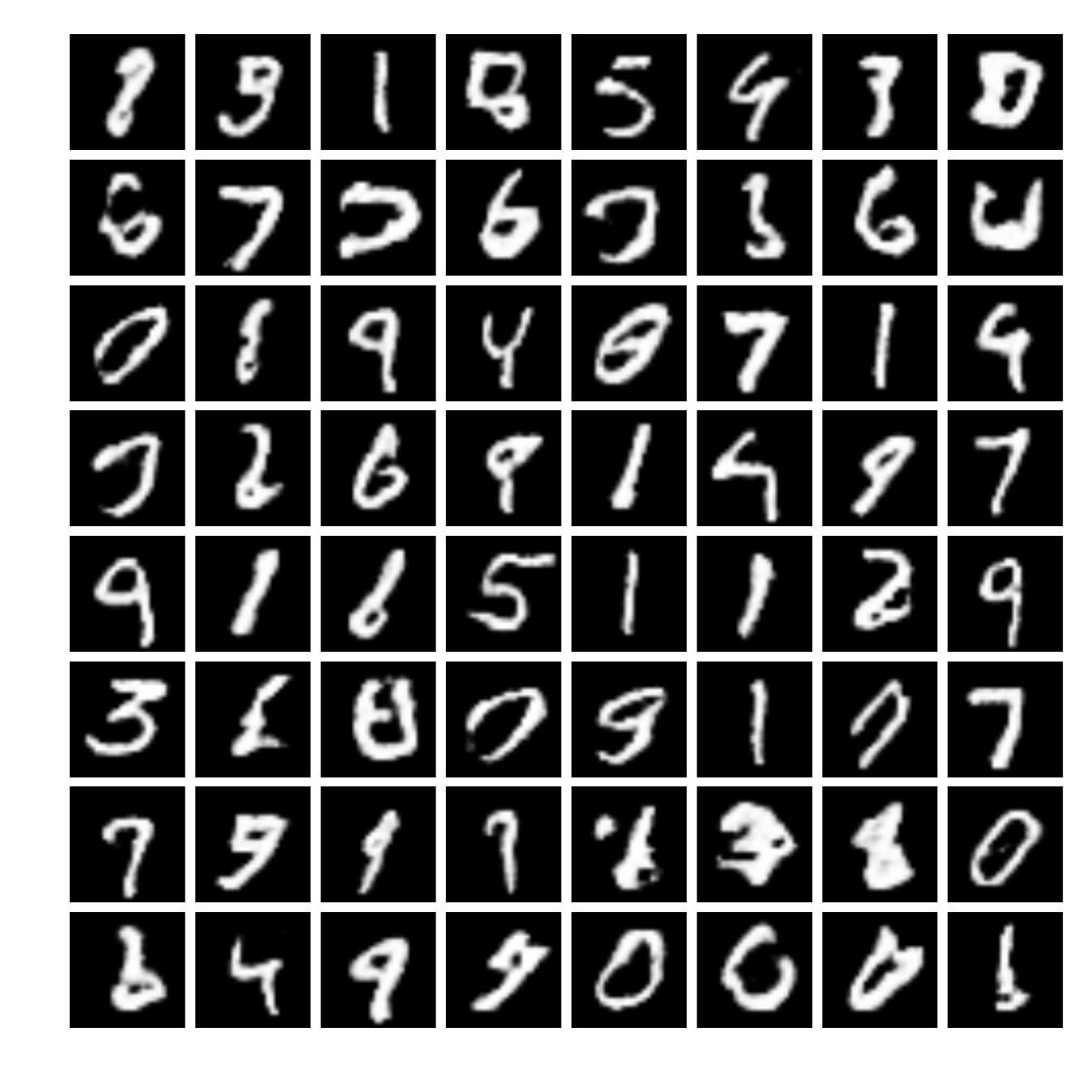}
	\caption{$K=2$}
\end{subfigure}
\begin{subfigure}{.49\linewidth}
	\centering
	\includegraphics[width=0.99\linewidth]{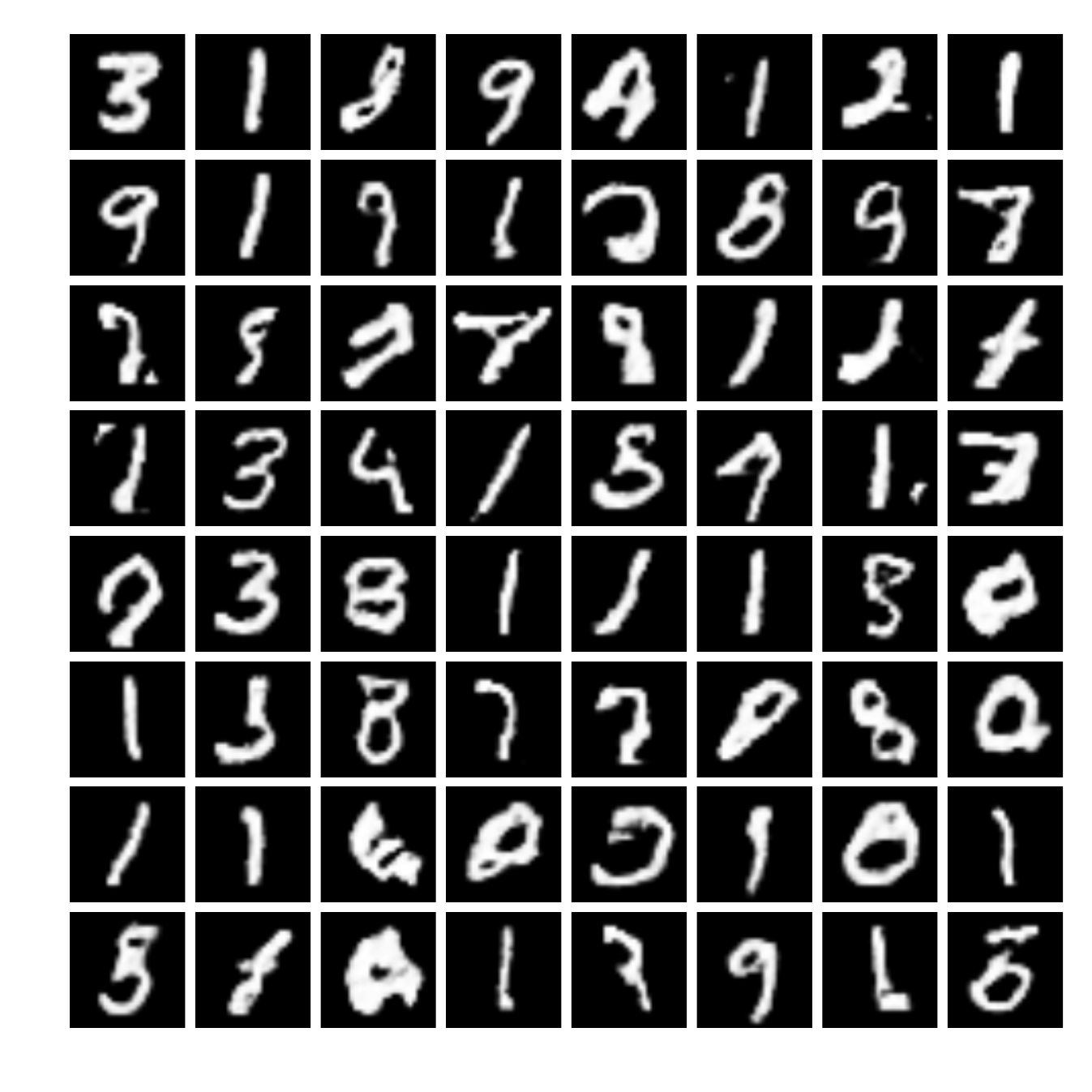}
	\caption{$K=5$}
\end{subfigure}
\begin{subfigure}{.49\linewidth}
	\centering
	\includegraphics[width=0.99\linewidth]{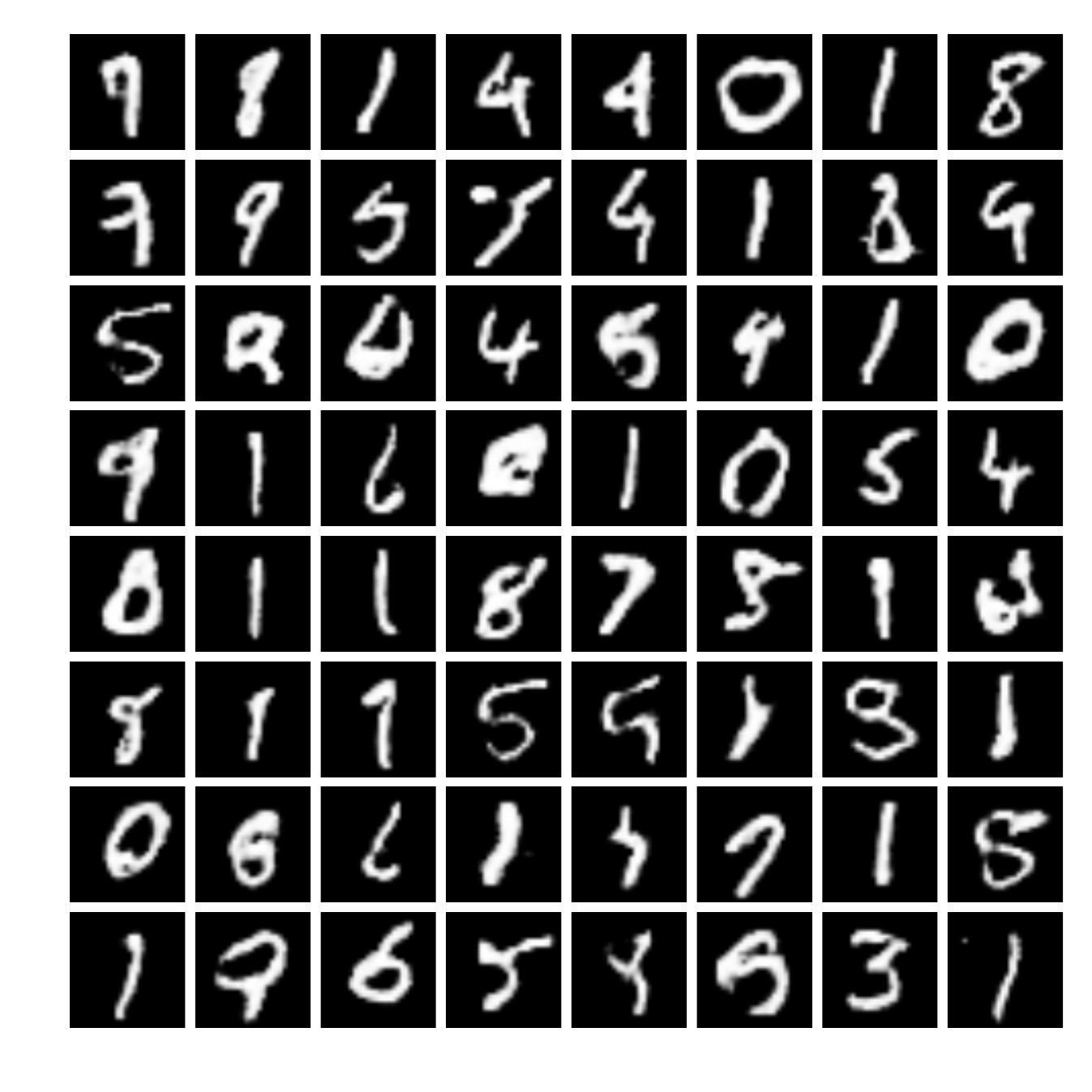}
	\caption{$K=10$}
\end{subfigure}
\caption{MNIST images generated using GAN after 10000 iterations,
trained with $K=1,2,5,10$.}
\label{fig:mnist}
\end{figure}

We also trained GANs to generate MNIST images with the $K$-beam method. 
The objective function is the same as the MoG experiments,
but the generator $G$ and the discriminator networks $D$ are more complex as
shown in Table~\ref{tbl:gan-mnist}.
 

The networks are trained with the batch size of 128 using the Adam optimizer
with the learning rate of $10^{-3}$.

Fig.~\ref{fig:mnist} shows typical training results for $K=1,2,5,10$ and $J=1$.
Images generated with a larger $K$ look slightly more natural than those with
a smaller $K$.  
However, an important difference is that GAN training often fails to converge to
a good solution due to ``mode collapsing'' \cite{nagarajan2017gradient} 
when $K$ is small, 
as observed by an abrupt change in the cost function during optimization.
The mode collapsing rarely happens with larger $K$'s such as $K$=10 with GAN-MNIST.
This difference in stability is not directly observable by qualitatively comparing
the best generated images from each setting, but it can be measured objectively by
average convergence and variance as shown in the figures of the main paper.

\begin{table}[htb]
\caption{Generator and discriminator networks for GAN-MNIST}
\begin{subtable}{1\linewidth}
	\centering
	\caption{Generator}
	\begin{tabular}{|l|l|}
	\hline
	Type & Size\\
	\hline
	Input & input dim=10 \\
	Fully connected & hidden nodes=7x7x64, ReLU \\
	Conv transpose & filter size=5x5x32, ReLU\\
	Conv transpose & filter size=5x5x1\\
	Sigmoid & output dim=28x28x1\\
	\hline
	\end{tabular}
\end{subtable}
\begin{subtable}{1\linewidth}
	\centering
	\caption{Discriminator}
	\begin{tabular}{|l|l|}
	\hline
	Type & Size\\
	\hline
	Input & input dim=28x28x1 \\
	Conv & filter size=5x5x16, ReLU\\
	Max pool & size=2x2, stride=2x2\\
	Conv & filter size=5x5x32, ReLU\\
	Max pool & size=2x2, stride=2x2\\
	Fully connected & hidden nodes=50, ReLU\\
	Fully connected & output dim=2	\\
	\hline
	\end{tabular}
\end{subtable}
\label{tbl:gan-mnist}
\end{table}

\end{document}